\documentclass{article}
\pdfoutput=1

\PassOptionsToPackage{numbers, compress, sort}{natbib}
     \usepackage[final]{neurips_2019}

\usepackage[utf8]{inputenc} 
\usepackage[T1]{fontenc}    
\usepackage{hyperref}       
\usepackage{url}            
\usepackage{booktabs}       
\usepackage{amsfonts}       
\usepackage{nicefrac}       
\usepackage{microtype}      
\usepackage{wrapfig}
\usepackage{enumitem}

\usepackage{todonotes}

\usepackage{amssymb}
\usepackage{multirow}
\usepackage{amsthm}
\usepackage{thmtools}
\usepackage{thm-restate}
\usepackage{tabu}
\usepackage{subcaption}

\theoremstyle{remark}
\newtheorem{assumption}{Assumption}

\usepackage[acronym,nowarn,section,nogroupskip,nonumberlist]{glossaries}
\glsdisablehyper{}
\newacronym{KL}{kl}{Kullback-Leibler}
\newacronym{SGD}{sgd}{stochastic gradient descent}
\newacronym{OED}{BOED}{Bayesian optimal experimental design}
\newacronym{MI}{MI}{mutual information}
\newacronym{EIG}{EIG}{expected information gain}
\newacronym{NMC}{NMC}{nested Monte Carlo}
\newacronym{MC}{MC}{Monte Carlo}
\newacronym{RMSE}{RMSE}{root mean squared error}

\usepackage{scalerel}
\usepackage{mleftright}
\usepackage{bm}

\DeclareMathOperator{\E}{{}\mathbb{E}}
\DeclareMathOperator{\1}{{}\mathbf{1}}

\DeclareMathOperator*{\argmin}{arg\,min} 
\DeclareMathOperator*{\argmax}{arg\,max} 

\renewcommand{\to}{\ensuremath{\rightarrow}}              

\newcommand{\eig}{\textup{EIG}}

\newcommand{\kl}[2]{\text{KL}\left(\,#1\,||\,#2\,\right)}
\newcommand{\entropy}[1]{H\!\left[#1\right]}
\newcommand{\expect}{\mathbb{E}}	
\newcommand{\R}{\mathbb{R}}

\newcommand{\qp}{q_p}
\newcommand{\qm}{q_m}
\newcommand{\ql}{q_\ell}
\newcommand{\Ipost}{\mathcal{L}_\text{post}}
\newcommand{\Imarg}{\mathcal{U}_\text{marg}}
\newcommand{\Ivnmc}{\mathcal{U}_\text{VNMC}}
\newcommand{\Iml}{\mathcal{I}_{\text{m}+\ell}}
\newcommand{\Idv}{\mathcal{L}_\text{DV}}
\newcommand{\ipost}{\hat{\mu}_\text{post}}
\newcommand{\imarg}{\hat{\mu}_\text{marg}}
\newcommand{\ivnmc}{\hat{\mu}_\text{VNMC}}
\newcommand{\iml}{\hat{\mu}_{\text{m}+\ell}}
\newcommand{\ilap}{\hat{\mu}_\text{laplace}}
\newcommand{\inmc}{\hat{\mu}_\text{NMC}}
\newcommand{\ilfire}{\hat{\mu}_\text{LFIRE}}
\newcommand{\idv}{\hat{\mu}_\text{DV}}
\newcommand{\lowerbound}{\mathcal{L}_{\text{post}}}
\newcommand{\upperbound}{\mathcal{U}_{\text{marg}}}
\newcommand{\upperboundvnmc}{\mathcal{U}_{\text{VNMC}}}
\newcommand{\Ialone}{\mu}
\newcommand{\ialone}{\hat{\mu}}
\newcommand{\iid}{\overset{\scriptstyle{\text{i.i.d.}}}{\sim}}

\usepackage{pifont}
\newcommand{\cmark}{\ding{51}}%
\newcommand{\xmark}{\ding{55}}%

\usepackage{titlesec}

\makeatletter
\newcommand{\settitle}{\@maketitle}
\makeatother

\allowdisplaybreaks

\title{Variational Bayesian Optimal Experimental Design}

\author{%
	Adam Foster\textsuperscript{\textdagger}\thanks{~~Part of this work was completed by AF during an internship with Uber AI Labs.} ~~ Martin Jankowiak\textsuperscript{\ddag} ~~ Eli Bingham\textsuperscript{\ddag} ~~ Paul Horsfall\textsuperscript{\ddag} \\ \textbf{Yee Whye Teh}\textsuperscript{\textdagger} ~~ \textbf{Tom Rainforth}\textsuperscript{\textdagger} ~~ \textbf{Noah Goodman}\textsuperscript{\ddag\textsection} \\
\textsuperscript{\textdagger}Department of Statistics, University of Oxford,
Oxford, UK \\ 
\textsuperscript{\ddag}Uber AI Labs, Uber Technologies Inc., San Francisco, CA, USA \\
\textsuperscript{\textsection}Stanford University, Stanford, CA, USA \\
  \texttt{adam.foster@stats.ox.ac.uk}
}

\begin{document}

\maketitle

\begin{abstract}
\acrfull{OED} is a principled framework for making efficient use of limited experimental resources. 
Unfortunately, its applicability is hampered by the difficulty of obtaining accurate estimates of the \acrfull{EIG} of an experiment.
To address this, we introduce several classes of fast \acrshort{EIG} estimators 
by building on ideas from amortized variational inference.
We show theoretically and empirically that these estimators can provide significant gains in speed and accuracy over previous approaches.
We further demonstrate the practicality of our approach on a number of end-to-end experiments.
\end{abstract}

\glsresetall

\section{Introduction}
\label{sec:introduction}

Tasks as seemingly diverse as designing a study to elucidate human cognition, selecting the next query point in an active learning loop, and designing online feedback surveys all constitute the same underlying problem: designing an experiment to maximize the information gathered.
\acrfull{OED} forms a powerful mathematical abstraction for tackling such problems~\citep{chaloner1995, lindley1956, sebastiani2000maximum, vincent2017} and has been successfully applied in numerous settings, including psychology \citep{myung2013}, Bayesian optimization \citep{hernandez2014}, active learning \citep{golovin2010}, bioinformatics \citep{vanlier2012}, and neuroscience \citep{shababo2013bayesian}.

In the~\acrshort{OED} framework, we construct a predictive model $p(y|\theta,d)$ for possible experimental outcomes $y$, given a design $d$ and a particular value of the parameters of interest $\theta$.  
We then choose the design that optimizes the \acrfull{EIG} in $\theta$ from running the experiment,
\begin{align}
\eig(d) & \triangleq
\mathbb{E}_{p(y|d)}\big[\entropy{p(\theta)}-\entropy{p(\theta | y, d)}\big],
\label{eq:EIG}
\end{align}
where $H[\cdot]$ represents the entropy and $p(\theta | y, d)\propto p(\theta)p(y|\theta,d)$ is the posterior resulting from running the experiment with design $d$ and observing outcome $y$.
In other words, we seek the design that, in expectation over possible experimental outcomes, most reduces the entropy of the posterior over our target latent variables.
If the predictive model is correct, this forms a design strategy that is (one-step) optimal from an information-theoretic viewpoint~\citep{lindley1972,sebastiani2000maximum}.

The \acrshort{OED} framework is particularly powerful in sequential contexts, where it allows the results of previous experiments to be used in guiding the designs for future experiments.
For example, as we ask a participant a series of questions in a psychology trial, we can use the information gathered from previous responses to ask more pertinent questions in the future, that will, in turn, return more information.
This ability to design experiments that are self-adaptive can substantially increase their efficiency: fewer iterations are required to uncover the same level of information.

In practice, however, the~\acrshort{OED} approach is often hampered by the difficulty of obtaining fast and high-quality estimates of the \acrshort{EIG}: due to the intractability of the posterior $p(\theta | y, d)$, it constitutes a nested expectation problem and so conventional \acrfull{MC} estimation methods cannot be applied~\citep{nmc}.
Moreover, existing methods for tackling nested expectations have, in general,
far inferior convergence rates than those for conventional expectations~\citep{lewi2009sequential,myung2013,rainforth2017thesis}.
For example, nested \acrshort{MC} (NMC) can only achieve, at best, a rate of $\mathcal{O}(T^{-1/3})$ in the total computational cost $T$~\citep{nmc}, compared with 
$\mathcal{O}(T^{-1/2})$ for conventional \acrshort{MC}.

To address this, we propose a variational \acrshort{OED} approach that sidesteps the double intractability of the \acrshort{EIG} in a principled manner and yields estimators with convergence rates in line with those for conventional estimation problems. To this end, we introduce four efficient and widely applicable variational estimators for the \acrshort{EIG}.
The different methods each present distinct advantages.
For example, two allow training with implicit likelihood models, while one allows for asymptotic consistency even when the variational family does not contain the target distribution.  

We theoretically confirm the advantages of our estimators, showing that they all have a convergence rate of $\mathcal{O}(T^{-1/2})$ when the variational family contains the target distribution.
We further verify their practical utility using a number of experiment design problems inspired by applications from science and industry, showing that they provide significant empirical gains in \acrshort{EIG} estimation over previous methods and that these gains lead, in turn, to improved end-to-end performance.

To maximize the space of potential applications and users for our estimators, we provide\footnote{Implementations of our methods are available at \texttt{http://docs.pyro.ai/en/stable/contrib.oed.html}. To reproduce the results in this paper, see \texttt{https://github.com/ae-foster/pyro/tree/vboed-reproduce}.} a general-purpose implementation of them in the probabilistic programming system Pyro \citep{pyro}, exploiting Pyro's first-class support for neural networks and variational methods.



\section{Background}
\label{sec:background}

The \acrshort{OED} framework is a model-based approach for choosing an experiment design $d$ in a manner 
that optimizes the information gained about some parameters of interest $\theta$ from the outcome $y$ of the experiment. For instance, we may wish to choose the question $d$ in a psychology trial to maximize the information gained about an underlying psychological property of the participant $\theta$ from their answer $y$ to the question.
In general, we adopt a Bayesian modelling framework with a prior $p(\theta)$ and a predictive model $p(y|\theta,d)$. The information gained about $\theta$ from running experiment $d$ and observing $y$ is the reduction in entropy from the prior to the posterior:
\begin{equation}
	\text{IG}(y,d) = \entropy{p(\theta)}-\entropy{p(\theta | y, d)}.
\end{equation}
At the point of choosing $d$, however, we are uncertain about the outcome. Thus, in order to define a metric to assess
the utility of the design $d$ we take the expectation of $\text{IG}(y,d)$ under
the marginal distribution over outcomes $p(y|d)=\mathbb{E}_{p(\theta)}[p(y|\theta,d)]$
as per~\eqref{eq:EIG}.  We can further rearrange this as
\begin{align}
\hspace{-1pt}	\eig(d) & 
	\label{eig:post}
	= \mathbb{E}_{p(y,\theta|d)} \left[ \log \frac{p(\theta | y, d)}{p(\theta)}  \right] 
	= \mathbb{E}_{p(y,\theta|d)} \left[ \log \frac{p(y, \theta | d)}{p(\theta)p(y|d)}  \right]
	= \mathbb{E}_{p(y,\theta|d)} \left[ \log \frac{p(y | \theta, d)}{p(y|d)} \right]
\end{align}
with the result that the \acrshort{EIG} can also be interpreted as the \acrlong{MI} between $\theta$ and $y$ given $d$,
or the epistemic uncertainty in $y$ averaged over the prior $p(\theta)$.
The Bayesian optimal design is defined as 
$d^* \triangleq \argmax_{d \in \mathcal{D}} \, \eig(d),$
where $\mathcal{D}$ is the set of permissible designs.

Computing the \acrshort{EIG} is challenging since neither $p(\theta | y, d)$ or $p(y|d)$ can, in general, be found in closed form.
Consequently, the integrand is intractable and conventional \acrshort{MC} methods are not  applicable.
One common way 
of getting around this is to employ a nested MC (NMC) estimator~\citep{myung2013,vincent2017}
\begin{equation}
\label{eq:nmc}
\hspace{-1pt} 
\inmc(d) \hspace{-1pt}\triangleq\hspace{-1pt} \frac{1}{N}\sum_{n=1}^{N}\log\frac{p(y_n|\theta_{n,0},d)}{\frac{1}{M}\sum_{m=1}^{M} p(y_n|\theta_{n,m},d)}
\,\,\,\,\,
\text{where} 
\,\,\,\,\,
\theta_{n,m} \! \overset{\scriptstyle{\text{i.i.d.}}}{\sim} \! p(\theta), 
\,
y_n \! \sim p(y|\theta=\theta_{n,0},d). 
\hspace{-7pt}
\end{equation}
\citet{nmc} showed that this estimator, which has a total computational cost $T=\mathcal{O}(NM)$, is consistent in the limit $N,M\to\infty$ with RMSE convergence rate $\mathcal{O}(N^{-1/2}+M^{-1})$, and that 
it is asymptotically optimal to set $M\propto \sqrt{N}$, yielding  an overall rate of $\mathcal{O}(T^{-1/3})$.

Given a base \acrshort{EIG} estimator, a variety of different methods can be used for the subsequent optimization over designs, including some specifically developed for BOED~\citep{amzal2006bayesian,muller2005simulation,rainforth2017thesis}.
In our experiments, we will adopt Bayesian optimization~\citep{snoek2012practical}, due to its sample efficiency, robustness to multi-modality, and ability to deal naturally with noisy objective evaluations.
However, we emphasize that our focus is on the base \acrshort{EIG} estimator and that our estimators can be used more generally with different optimizers.

The static design setting we have implicitly assumed thus far in our discussion  can be generalized to sequential contexts, in which we design $T$ experiments $d_1, ..., d_T$ with outcomes $y_1, ..., y_T$. We assume experiment outcomes are conditionally independent given the latent variables and designs, i.e.
\begin{equation}
	\label{eq:sequential}
	p(y_{1:T},\theta|d_{1:T}) = p(\theta)\prod_{t=1}^T p(y_t|\theta,d_t).
\end{equation}
Having conducted experiments $1, ..., t-1$, we can design $d_t$ by 
incorporating data in the standard Bayesian fashion:
at experiment iteration $t$, we replace
the prior $p\left(\theta\right)$ in \eqref{eig:post} with 
$p\left(\theta | d_{1:t-1}, y_{1:t-1}\right)$, the posterior conditional on the first $t-1$ designs and outcomes.
We can thus conduct an adaptive sequential experiment in which we optimize the choice of the design $d_t$ at each iteration.



\section{Variational Estimators}
\label{sec:method}

Though consistent, the convergence rate of the \acrshort{NMC} estimator is prohibitively slow for many practical problems.  
As such, \acrshort{EIG} estimation often becomes the bottleneck for \acrshort{OED}, particularly in sequential experiments where the \acrshort{OED} calculations must be fast enough to operate in real-time.

In this section we show how ideas from amortized variational inference \citep{dayan1995helmholtz,kingma2014auto,rezende2014stochastic,stuhlmuller2013learning} can be used to sidestep the double intractability of the \acrshort{EIG},
yielding estimators with much faster convergence rates thereby alleviating the \acrshort{EIG} bottleneck.
A key insight for realizing why such fundamental gains can be made is that the \acrshort{NMC} estimator is inefficient because a \emph{separate} estimate of the integrand in \eqref{eig:post} is made for each $y_n$. 
The variational approaches we introduce instead look to directly learn a \emph{functional approximation}---for 
example, an approximation of $y\mapsto p(y|d)$---and then evaluate this approximation at multiple points to estimate the integral, thereby allowing information to be shared across different values of $y$.  
If $M$ evaluations are made in learning the approximation, the total computational cost is now $T=\mathcal{O}(N+M)$,  yielding substantially improved convergence rates.

\paragraph{Variational posterior $\ipost$}
\label{sec:ipost}

Our first approach, which we refer to as the variational posterior estimator $\ipost$, is based on learning an amortized approximation $\qp(\theta|y,d)$ to the posterior $p(\theta | y,d)$ and then using this to estimate the \acrshort{EIG}:
\begin{align}	
\eig(d) \approx \lowerbound(d) \triangleq \mathbb{E}_{p(y, \theta | d)}\left[ \log \frac{\qp(\theta | y, d)}{p(\theta)} \right]
\approx \ipost(d) \triangleq \frac{1}{N} \sum_{n=1}^{N} \log \frac{ \qp(\theta_n|y_n,d)}{p(\theta_n)},
\label{eq:posteriorbound}	
\end{align}
where $y_n, \theta_n \iid p(y, \theta | d)$ and $\ipost(d)$ is a MC estimator of $\lowerbound(d)$.
We draw samples of $p(y,\theta|d)$ by sampling $\theta \sim p(\theta)$ and then $\,y|\theta \sim p(y|\theta,d)$.
We can think of this approach as amortizing the cost of the inner expectation, instead of running inference separately for each $y$.

To learn a suitable $\qp(\theta|y,d)$, we show in Appendix~\ref{sec:app:method} that $\lowerbound(d)$ forms a variational lower bound $\eig(d) \ge \lowerbound(d)$ that is tight if and only if $\qp(\theta|y,d) = p(\theta|y,d)$. \citet{ba} used this bound to estimate mutual information in the context of transmission over noisy channels, but the connection to experiment design has not previously been made.

This result means we can learn $\qp(\theta | y, d)$ by introducing a family of variational distributions $\qp(\theta | y, d, \phi)$  parameterized by $\phi$ and then maximizing the bound with respect to $\phi$:
\begin{align}
\label{eq:maxpost}
&\phi^* = \argmax_\phi \mathbb{E}_{p(y, \theta | d)}\left[ \log \frac{\qp(\theta | y, d, \phi)}{p(\theta)} \right], 
\quad \quad \eig(d) \approx \lowerbound(d;\phi^*).
\end{align}
Provided that we can generate samples from the model, this maximization can be performed using stochastic gradient methods \citep{robbins1951stochastic} and the unbiased gradient estimator
\begin{align}
\label{eq:gradl}
\begin{split}
\nabla_{\phi} \lowerbound(d;\phi) &\approx \tfrac{1}{S}\sum\nolimits_{i=1}^S \nabla_\phi \log \qp(\theta_i|y_i, d, \phi) \quad 
\text{where} \quad y_i, \theta_i \iid p(y, \theta | d),
\end{split}
\end{align}
and we note that no reparameterization is required as $p(y,\theta|d)$ is independent of $\phi$.  After $K$ gradient steps we obtain variational parameters $\phi_K$ that approximate $\phi^*$, which we use to compute a corresponding $\eig$ estimator by constructing a \acrshort{MC} estimator for $\lowerbound(d;\phi)$ as per~\eqref{eq:posteriorbound} with $\qp(\theta_n|y_n,d)=\qp(\theta_n|y_n,d,\phi_K)$.  Interestingly, the tightness of $\lowerbound(d)$ turns out to be equal to the expected \emph{forward} KL divergence\footnote{See Appendix~\ref{sec:app:method} for a proof. A comparison with the reverse KL divergence can be found in Appendix~\ref{sec:reverseforward}.} $\mathbb{E}_{p(y|d)} \left[\text{KL}\left(p(\theta | y,d)||\qp(\theta | y, d, \phi)\right)\right]$  so we can view this approach as learning an amortized proposal by minimizing this expected KL divergence.

\paragraph{Variational marginal $\imarg$}
\label{sec:imarg}
In some scenarios, $\theta$ may be high-dimensional, making it difficult to train a good variational posterior approximation.
An alternative approach that can be attractive in such cases is to instead learn an approximation $\qm(y|d)$ 
to the marginal density $p(y|d)$ and substitute this into the final form of the \acrshort{EIG} in~\eqref{eig:post}.
As shown in Appendix~\ref{sec:app:method}, this yields an \emph{upper bound}
\begin{align}
	\label{eq:margbound}
	\eig(d) 
	\le 
	\upperbound(d) 
	\triangleq 
	\mathbb{E}_{p(y, \theta | d)}\left[ \log \frac{p(y | \theta, d)}{\qm(y|d)} \right]
	\approx
	\imarg(d) \triangleq \frac{1}{N} \sum_{n=1}^{N} \log \frac{p(y_n | \theta_n, d)}{\qm(y_n | d)},
\end{align}
where again $y_n, \theta_n \iid p(y, \theta | d)$ and the bound is tight when $\qm(y|d)=p(y|d)$.  
Analogously to $\ipost$, we can learn $\qm(y|d)$ by introducing a variational family $\qm(y|d,\phi)$ and then performing stochastic gradient descent to \emph{minimize} $\upperbound(d,\phi)$.
As with $\ipost$, this bound was studied in a mutual information context \cite{poole2018variational}, but it has not been utilized for \acrshort{OED} before.

\paragraph{Variational \acrshort{NMC} $\ivnmc$}

As we will show in Section~\ref{sec:pointwise}, $\ipost$ and $\imarg$ can provide substantially faster convergence rates than \acrshort{NMC}.
However, this comes at the cost of converging towards a biased estimate if the variational family does not contain the target distribution. To address this, we propose another EIG estimator, $\ivnmc$, which allows one to trade-off resources between the fast learning of a biased estimator permitted by variational approaches, and the ability of NMC to eliminate this bias.\footnote{In Appendix~\ref{sec:cv}, we describe a method using $\qm(y|d)$ as a control variate that can also eliminate this bias and lower the variance of \acrshort{NMC}, requiring additional assumptions about the model and variational family.}

We can think of the \acrshort{NMC} estimator as approximating $p(y|d)$ using $M$ samples from the prior. At a high-level, $\ivnmc$ is based around learning a proposal $q_v(\theta | y, d)$ and then using samples from this proposal to make an importance sampling estimate of $p(y|d)$, potentially requiring far fewer samples than \acrshort{NMC}. Formally, it is based around a bound that can be arbitrarily tightened, namely
\begin{align}
\eig(d) &\le \expect\left[ \log p(y|\theta_0,d) - \log \frac{1}{L} \sum_{\ell=1}^L \frac{p(y,\theta_\ell|d)}{q_v(\theta_\ell|y,d)} \right] \triangleq \upperboundvnmc(d,L)
\label{eq:ivnmc}
\end{align}
where the expectation is taken over $y,\theta_{0:L} \sim p(y,\theta_0|d)\prod_{\ell=1}^L q_v(\theta_\ell|y,d)$, which corresponds to one sample $y,\theta_0$ from the  model and $L$ samples from the approximate posterior conditioned on $y$.
To the best of our knowledge, this bound has not previously been studied in the literature.
As with $\ipost$ and $\imarg$, we can minimize this bound to train a variational approximation $q_v(\theta|y,d,\phi)$.
Important features of $\upperboundvnmc(d,L)$ are summarized in the following lemma; see Appendix~\ref{sec:app:method} for the proof.
\begin{restatable}{lemma}{lemvnmc}
For any given model $p(\theta)p(y|\theta,d)$ and valid  $q_v(\theta | y,d)$,
\begin{enumerate}[leftmargin=0.2in]
	\setlength\itemsep{0em}
	\item $\eig(d) = \lim_{L\to\infty} \upperboundvnmc(d,L) \le \upperboundvnmc(d,L_2) \le \upperboundvnmc(d,L_1) \quad \forall L_2 \ge L_1 \ge 1$,
	\item  $\upperboundvnmc(d,L)=\eig(d) \,\,\,\, \forall L\ge1 \,\,\,\,$ if $\,\,\,\, q_v(\theta | y, d) = p(\theta | y, d) \,\,\,\, \forall y, \theta$,
	\item $\upperboundvnmc(d,L)-\eig(d) \!=\! \E_{p(y|d)}\! \left[\textnormal{KL}\left(\prod_{\ell=1}^L q_v(\theta_\ell|y,d) \big|\big| \frac{1}{L}\sum_{\ell=1}^L p(\theta_\ell|y,d)\prod_{k\ne\ell}q_v(\theta_k|y,d) \right) \right]$
\end{enumerate}
\label{lemma:vnmc}
\end{restatable}
\vspace{-5pt}
Like the previous bounds, the VNMC bound is tight when $q_v(\theta|y,d) = p(\theta|y,d)$. Importantly, the bound is also tight as $L\to\infty$, even for imperfect $q_v$. This means we can obtain asymptotically unbiased \acrshort{EIG} estimates even when the true posterior is not contained in the variational family.

Specifically, we first train $\phi$ using $K$ steps of stochastic gradient on $\upperboundvnmc(d,L)$ with some fixed $L$. To form a final \acrshort{EIG} estimator, however, we use a \acrshort{MC} estimator of $\upperboundvnmc(d,M)$ where typically $M \gg L$. This final estimator is a \acrshort{NMC} estimator that is consistent as $N,M\to \infty$ with $\phi_K$ fixed
\begin{align}
\label{eq:vnmc}
\begin{split}
\ivnmc(d) \triangleq \frac{1}{N}\sum_{n=1}^{N} \left(\log p(y_n|\theta_{n,0},d)
-\log \frac{1}{M}\sum_{m=1}^{M} \frac{p(y_n,\theta_{n,m}|d)}{q_v(\theta_{n,m}|y_n,d,\phi_K)}\right)
\end{split}
\end{align}
where $\theta_{n,0} \iid p(\theta)$, $y_n \sim p(y|\theta=\theta_{n,0},d)$ and $\theta_{n,m} \sim q_v (\theta | y=y_n,d,\phi_K)$. In practice, performance is greatly enhanced when the proposal $q_v$ is a good, if inexact, approximation to the posterior. 
This significantly improves upon traditional $\inmc$, which sets $q_v(\theta|y,d) = p(\theta)$ in \eqref{eq:vnmc}.

\paragraph{Implicit likelihood and $\iml$}

So far we have assumed that 
we can evaluate $p(y|\theta,d)$ pointwise. 
However, many models of interest have \textit{implicit likelihoods} from which we can draw samples, but not evaluate directly.
For example, models with nuisance latent variables $\psi$ (such as a random effect models)
are implicit likelihood models because $p(y|\theta, d) = \mathbb{E}_{p(\psi|\theta)}\left[p(y|\theta, \psi, d)\right]$ is intractable, but can still be straightforwardly sampled from.

In this setting, $\ipost$ is applicable without modification because it only requires samples from $p(y|\theta,d)$ and \emph{not} evaluations of this density.
Although $\imarg$ is not directly applicable in this setting, it can be modified to accommodate implicit likelihoods.
Specifically, we can utilize \emph{two} approximate densities: $\qm(y|d)$ for the marginal and $\ql(y|\theta,d)$ for the likelihood. We then form the approximation 
\begin{equation}
	\label{eq:marginallikelihoodest}
\eig(d) \approx \Iml(d)  \triangleq  \mathbb{E}_{p(y,\theta|d)}\left[ \log \frac{\ql(y|\theta,d)}{\qm(y|d)} \right] \approx \iml(d) \triangleq \frac{1}{N} \sum_{n=1}^{N} \log \frac{\ql(y_n | \theta_n, d)}{\qm(y_n | d)}.
\end{equation}
Unlike the previous three cases, $\Iml(d)$ is not a bound on $\eig(d)$, meaning it is not immediately clear how to train $\qm(y|d)$ and $\ql(y|\theta,d)$ to achieve an accurate \acrshort{EIG} estimator. The following lemma shows that we can bound the \acrshort{EIG} estimation \textit{error} of $\Iml$. The proof is in Appendix~\ref{sec:app:method}.
\begin{restatable}{lemma}{lemiml}
For any given model $p(\theta)p(y|\theta,d)$ and valid $\qm(y|d)$ and $\ql(y|\theta,d)$, we have
\begin{equation}
|\Iml(d) - \eig(d)| \le - \mathbb{E}_{p(y,\theta|d)}[\log \qm(y|d) + \log \ql(y|\theta,d)] + C,
\label{eq:mlbound}
\end{equation}
where $C=-H[p(y|d)] - \mathbb{E}_{p(\theta)}\left[ H(p(y|\theta,d) \right]$ does not depend on $\qm$ or $\ql$. Further, the RHS of \eqref{eq:mlbound} is 0 if and only if $\qm(y|d)=p(y|d)$ and $\ql(y|\theta,d)=p(y|\theta,d)$ for almost all $y,\theta$.
\label{lemma:ml}
\end{restatable}
This lemma implies that we can learn $\qm(y|d)$ and $\ql(y|\theta,d)$ by maximizing $\mathbb{E}_{p(y,\theta|d)}[\log \qm(y|d) + \log \ql(y|\theta,d)]$ using stochastic gradient ascent, and substituting these learned approximations into \eqref{eq:marginallikelihoodest} for the final \acrshort{EIG} estimator. 
To the best of our knowledge, this approach has not previously been considered in the literature.
We note that, in general, $\qm$ and $\ql$ are learned separately and there need not be any weight sharing between them. See Appendix~\ref{sec:app:iml} for a discussion of the case when we  couple $\qm$ and $\ql$ so that $\qm(y|d) = \E_{p(\theta)}[\ql(y|\theta,d)]$.

\paragraph{Using estimators for sequential \acrshort{OED}}

In sequential settings, we also need to consider the implications of replacing $p(\theta)$ in the \acrshort{EIG} with $p(\theta | d_{1:t-1}, y_{1:t-1})$. At first sight, it appears that, while $\imarg$ and $\iml$ only require samples from $p(\theta | d_{1:t-1}, y_{1:t-1})$, $\ipost$ and $\ivnmc$ also require its density to be evaluated, a potentially severe limitation.  Fortunately, we can, in fact, avoid evaluating this posterior density.  We note that, from \eqref{eq:sequential}, we have $p(\theta|y_{1:t-1},d_{1:t-1}) = p(\theta)\prod_{i=1}^{t-1}p(y_i|\theta,d_i)/p(y_{1:t-1}|d_{1:t-1})$. Substituting this into the integrand of \eqref{eq:posteriorbound} gives
\begin{equation}
	\Ipost(d_t) = \E_{p(\theta|y_{1:t-1},d_{1:t-1})p(y_t|\theta,d_t)}\left[ \log \frac{\qp(\theta|y_t,d_t)}{p(\theta)\prod_{i=1}^{t-1}p(y_i|\theta,d_i)} \right] + \log p(y_{1:t-1}|d_{1:t-1})
\end{equation}
where $p(\theta)\prod_{i=1}^{t-1}p(y_i|\theta,d_i)$ can be evaluated exactly and the additive constant $\log p(y_{1:t-1}|d_{1:t-1})$ does not depend on the new design $d_t$, $\theta$, or any of the variational parameters, and so can be safely ignored. Making the same substitution in \eqref{eq:vnmc} shows that we can also estimate $\Ivnmc(d_t, L)$ up to a constant, which can then be similarly ignored.  As such, any inference scheme for sampling $p(\theta | d_{1:t-1}, y_{1:t-1})$, approximate or exact, is compatible with all our approaches.




\paragraph{Selecting an estimator}
\begin{wraptable}{r}{0.52\columnwidth}
	\begin{center}
		\vspace{-16pt}
		\caption{Summary of \Acrshort{EIG} estimators. Baseline methods are explained in Section~\ref{sec:experimentseig}.}
		\label{tab:estimators}
		{\renewcommand{\arraystretch}{1}
			\setlength\tabcolsep{4pt} 
			\begin{tabular}{rlcccl}
			\hline
				&& \small Implicit & \small Bound & \small Consistent & \small Eq.  \\
				\hline
				\multirow{4}{*}{\rotatebox{90}{\small Ours~}} & \small $\ipost$ & \cmark & \small Lower & \xmark & \eqref{eq:posteriorbound} \\
				& \small $\imarg$ & \xmark & \small Upper & \xmark & \eqref{eq:margbound} \\
				& \small $\ivnmc$  & \xmark & \small Upper & \cmark & \eqref{eq:vnmc} \\
				& \small $\iml$  & \cmark & \xmark & \xmark & \eqref{eq:marginallikelihoodest} \\
				\hline
								\multirow{4}{*}{\rotatebox{90}{\small Baseline~}} & \small $\inmc$ & \xmark & \small Upper & \cmark & \eqref{eq:nmc} \\
				& \small $\ilap$ & \xmark & \xmark & \xmark & \eqref{eq:laplace} \\
				& \small $\ilfire$ & \cmark & \xmark & \xmark & \eqref{eq:lfire} \\
				& \small $\idv$ & \cmark & \small Lower & \xmark & \eqref{eq:dv} \\
				\hline
			\end{tabular}}
		\end{center}
		\vspace{-18pt}
	\end{wraptable}
Having proposed four estimators, we briefly discuss how to choose between them in practice. 
For reference, a summary of our estimators is given
in Table~\ref{tab:estimators}, along with several baseline approaches.
First, $\imarg$ and $\iml$ rely on approximating a distribution over $y$; $\ipost$ and $\ivnmc$ approximate distributions over $\theta$. We may prefer the former two estimators if $\text{dim}(y) \ll \text{dim}(\theta)$ as it leaves us with a simpler density estimation problem, and vice versa. Second, $\imarg$ and $\ivnmc$ require an explicit likelihood whereas $\ipost$ and $\iml$ do not. If an explicit likelihood is available, it typically makes sense to use it---one would never use $\iml$ over $\imarg$ for example. Finally, if the variational families do not contain the target densities, $\ivnmc$ is the only method guaranteed to converge to the true $\eig(d)$ in the limit as the computational budget increases.
So we might prefer $\ivnmc$ when computation time and cost are not constrained.



\section{Convergence rates}
\label{sec:pointwise}

We now investigate the convergence of our estimators.
We start by breaking the overall error down into three terms: I) variance in \acrshort{MC} estimation of the bound;
II) the gap between the bound and the tightest bound possible given the variational family;
and III) the gap between the tightest possible bound and $\eig(d)$.
With variational \acrshort{EIG} approximation $\mathcal{B}(d)\in\{\Ipost(d),\,\Imarg(d),\,\Ivnmc(d,L),\,\Iml(d)\}$, optimal variational parameters $\phi^*$, learned variational parameters $\phi_K$ after $K$ stochastic gradient iterations, and \acrshort{MC} estimator $\hat{\mu}(d,\phi_K)$ we have, by the triangle inequality,
\begin{subequations}
\begin{align}
\nonumber
\left\| \hat{\mu}(d,\phi_K) \!-\! \eig(d) \right \|_{2}
\le \underbrace{\left\|  \hat{\mu}(d,\phi_K) \!-\! \mathcal{B}(d,\phi_K)\right\|_{2}  }_{\rm I}
+ \underbrace{ \left\| \mathcal{B}(d,\phi_K)\!-\! \mathcal{B}(d,\phi^*) \right\|_{2}   }_{\rm II}
+ \underbrace{ \left| \mathcal{B}(d,\phi^*) \!-\! \eig(d) \right|  }_{\rm III}
\end{align}
\end{subequations}
where we have used the notation $\left\|X\right\|_{2}\triangleq\sqrt{\expect\left[X^2\right]}$ to denote the $L^2$ norm of a random variable.

By the weak law of large numbers, term I scales as $N^{-1/2}$ and can thus be arbitrarily reduced by taking more \acrshort{MC} samples.
Provided that our stochastic gradient scheme converges, term II can be reduced by increasing the number of stochastic gradient steps $K$.
Term III, however, is a constant that can only be reduced by expanding the variational family (or increasing $L$ for $\ivnmc$). 
Each approximation $\mathcal{B}(d)$ thus converges to a biased estimate of the $\eig(d)$, namely $\mathcal{B}(d,\phi^*)$.
As established by the following Theorem, if we set $N\propto K$, the rate of convergence to this biased estimate is $\mathcal{O}(T^{-1/2})$, where $T$ represents the total computational cost, with $T=\mathcal{O}(N+K)$.
\begin{restatable}{theorem}{themain}
	\label{the:main}
	Let $\mathcal{X}$ be a measurable space and $\Phi$ be a convex subset of a finite dimensional inner product space. Let $X_1, X_2,...$ be i.i.d.~random variables taking values in $\mathcal{X}$ and $f:\mathcal{X}\times \Phi \to \R$ 
	be a measurable function. Let
		$$\Ialone(\phi) \triangleq \expect[f(X_1, \phi)] \approx \ialone_N(\phi) \triangleq \frac{1}{N}\sum\nolimits_{n=1}^N f(X_n, \phi)$$
	and suppose that $\sup_{\phi \in \Phi} \|f(X_1, \phi)\|_{2} < \infty$. Then $
		\sup_{\phi \in \Phi} \|\ialone_N(\phi) - \Ialone(\phi)\|_{2} = \mathcal{O}(N^{-1/2})$.	
	Suppose further that Assumption~\ref{ass:sgd} in Appendix~\ref{sec:app:theory} holds and that $\phi^*$ is the unique minimizer of $\Ialone$. After $K$ iterations of the Polyak-Ruppert averaged stochastic gradient descent algorithm of \cite{moulines2011} with gradient estimator $\nabla_\phi f(X_t, \phi)$, we have $\|\Ialone(\phi_K) - \Ialone(\phi^*)\|_{2} = \mathcal{O}(K^{-1/2})$
	and, combining with the first result,$$\|\ialone_N(\phi_K) - \Ialone(\phi^*)\|_{2} = \mathcal{O}(N^{-1/2}+K^{-1/2}) 
		= \mathcal{O}(T^{-1/2}) \text{ if } N\propto K.$$
\end{restatable}
\vspace{-5pt}
The proof relies on standard results from \acrshort{MC} and stochastic optimization theory; see Appendix~\ref{sec:app:theory}.
We note that the assumptions required for the latter, though standard in the literature, are strong.
In practice, $\phi$ can converge to a local optimum $\phi^\dag$, rather than the global optimum $\phi^*$, introducing an additional asymptotic bias $\left| \mathcal{B}(d,\phi^\dag) - \mathcal{B}(d,\phi^*) \right|$ into term III.

Theorem~\ref{the:main} can be applied directly to $\imarg$, $-\ipost$, and $\ivnmc$ (with fixed $M=L$), showing that they converge respectively to $\upperbound(d,\phi^*)$, $-\lowerbound(d,\phi^*)$, and $\upperboundvnmc(d,L,\phi^*)$ at a rate $=\mathcal{O}(T^{-1/2})$ if $N\propto K$ and the assumptions are satisfied. For $\iml$, we combine Theorem~\ref{the:main} and Lemma~\ref{lemma:ml} to obtain the same $\mathcal{O}(T^{-1/2})$ convergence rates; see the supplementary material for further details.

The key property of $\ivnmc$ is that we need not set $M=L$ and can remove the asymptotic bias by increasing $M$ with $N$.
We begin by training $\phi$ with a fixed value of $L$, decreasing the error term $\|\Ivnmc(d, L, \phi_K) - \Ivnmc(d, L, \phi^*)\|_2$ at the fast rate $\mathcal{O}(K^{-1/2})$ until $|\Ivnmc(d, L, \phi^*) - \eig(d)|$ becomes the dominant error term. At this point, we start to increase $N, M$. Using the NMC convergence results discussed in Sec.~\ref{sec:background}, if we set $M \propto \sqrt{N}$, then $\ivnmc$ converges to $\eig(d)$ at a rate $\mathcal{O}((NM)^{-1/3})$. Note that the total cost of the $\ivnmc$ estimator is $T=\mathcal{O}(KL+NM)$, where typically $M\gg L$. The first stage, costing $KL$, is fast variational training of an amortized importance sampling proposal for $p(y|d) = \expect_{p(\theta)}[p(y|\theta,d)]$. The second stage, costing $NM$, is slower refinement to remove the asymptotic bias using the learned proposal in an NMC estimator. One can think of the standard NMC approach as a special case of $\ivnmc$ in which we naively choose $p(\theta)$ as the proposal. That is, standard NMC skips the first stage and hence does not benefit from the improved convergence rate of learning an amortized proposal. It typically requires a much higher total cost to achieve the same accuracy as VNMC.

%



\section{Related work}
\label{sec:experimentseig}
  
We briefly discuss alternative approaches to \acrshort{EIG} estimation for \acrshort{OED} that will form our baselines for empirical comparisons. The \textbf{\Acrfull{NMC}} baseline was introduced in Sec.~\ref{sec:background}. Another established approach is to use a \textbf{Laplace approximation} to the posterior \citep{lewi2009sequential,long2013}; this approach is fast but is limited to continuous variables and can exhibit large bias. \citet{kleinegesse2018efficient} recently suggested an implicit likelihood approach based on the Likelihood-Free Inference by Ratio Estimation \textbf{(LFIRE)} method of~\citet{dutta2016likelihood}. We also consider a method based on the \textbf{Donsker-Varadhan (DV)} representation of the KL divergence \citep{donsker1975asymptotic} as used by \citet{mine} for mutual information estimation. 
Though not previously considered in \acrshort{OED}, we include it as a baseline for illustrative purposes. For a full discussion of the DV bound and a number of other variational bounds used in deep learning, we refer to  the recent work of \citet{poole2018variational}.
For further discussion of related work, see Appendix~\ref{sec:app:relatedwork}.



\section{Experiments}
\label{sec:experiments}

\subsection{\acrshort{EIG} estimation accuracy}
\label{sec:eigaccuracy}

\setlength{\tabcolsep}{2pt}
\begin{table*}[t]
	\begin{center}
		\caption{Bias squared and variance from 5 runs, averaged over designs, of EIG estimators applied to four benchmarks. We use - to denote that a method does not apply and $*$ when it is superseded by other methods. Bold indicates the estimator with the lowest empirical mean squared error.
			}
		\label{tab:abserrors}
		\begin{tabu}{lcccccccc} 
		\hline
			& \multicolumn{2}{c}{\small \textbf{A/B test}} & \multicolumn{2}{c}{\small \textbf{Preference}}  & \multicolumn{2}{c}{\small \textbf{Mixed effects}}  & \multicolumn{2}{c}{\small \textbf{Extrapolation}} \\  
			& \small Bias$^2$ & \small Var  & \small Bias$^2$ & \small Var  & \small Bias$^2$ & \small Var  & \small Bias$^2$ & \small Var   \\
			\tabucline[1pt]{-}
			\small $\ipost$ & $\scriptstyle{1.33\times 10^{-2}}$ & $\scriptstyle{7.15\times 10^{-3}}$ & $\scriptstyle{4.26\times 10^{-2}}$ & $\scriptstyle{8.53\times 10^{-3}}$ & $\scriptstyle{2.34\times 10^{-3}}$ & $\scriptstyle{2.92\times 10^{-3}}$ & $\scriptstyle{1.24\times 10^{-4}}$ & $\scriptstyle{5.16\times 10^{-5}}$    \\ 
			\small $\imarg$ & $\scriptstyle{7.45\times 10^{-2}}$   & $\scriptstyle{6.41\times 10^{-3}}$  & $\mathbf{\scriptstyle{1.10\times 10^{-3}}}$ & $\mathbf{\scriptstyle{1.99\times 10^{-3}}}$ &\small - & \small - & \small - & \small -    \\   
			\small $\ivnmc$ & $\scriptstyle{3.44\times 10^{-3}}$ &  $\scriptstyle{3.38\times 10^{-3}}$   & $\scriptstyle{4.17\times 10^{-3}}$  & $\scriptstyle{9.04\times 10^{-3}}$ & \small - & \small - & \small - & \small -  \\
			\small $\iml$ & \small $*$   &\small $*$  & \small $*$ & \small $*$ & $\mathbf{\scriptstyle{3.06\times 10^{-3}}}$ & $\mathbf{\scriptstyle{5.94\times10^{-5}}}$ &  $\mathbf{\scriptstyle{6.90\times 10^{-6}}}$ &  $\mathbf{\scriptstyle{1.84\times 10^{-5}}}$    \\   
			\hline
			\small $\inmc$ & $\scriptstyle{4.70\times 10^0}$ & $\scriptstyle{3.47\times 10^{-1}}$ &  $\scriptstyle{7.60\times 10^{-2}}$ & $\scriptstyle{8.36\times 10^{-2}}$ & \small - & \small -  & \small - & \small -     \\ 
			\small $\ilap$ & \small $\mathbf{\scriptstyle{1.92\times 10^{-4}}}$ & $\mathbf{\scriptstyle{1.47\times 10^{-3}}}$ & $\scriptstyle{8.42\times 10^{-2}}$ & $\scriptstyle{9.70\times 10^{-2}}$ & \small - & \small -  & \small - & \small -   \\ 
			\small $\ilfire$ & $\scriptstyle{2.29\times10^0}$ &  $\scriptstyle{6.20\times10^{-1}}$  & $\scriptstyle{1.30\times 10^{-1}}$ & $\scriptstyle{1.41\times 10^{-2}}$ & $\scriptstyle{1.41\times 10^{-1}}$ & $\scriptstyle{6.67\times 10^{-2}}$ & \small - & \small - \\ 
			\small $\idv$ & $\scriptstyle{4.34\times 10^0}$ &  $\scriptstyle{8.85\times 10^{-1}}$ & $\scriptstyle{9.23\times 10^{-2}}$ & $\scriptstyle{8.07\times 10^{-3}}$ & $\scriptstyle{9.10 \times 10^{-3}}$ & $\scriptstyle{5.56\times 10^{-4}}$ & $\scriptstyle{7.84\times 10^{-6}}$ & $\scriptstyle{4.11\times 10^{-5}}$ \\
			\hline
		\end{tabu}
	\end{center}
	\vspace{-5pt}
\end{table*}
\setlength{\tabcolsep}{6pt}

We begin by benchmarking our \acrshort{EIG} estimators against the aforementioned baselines.
We consider four experiment design scenarios inspired by applications of Bayesian data analysis in science and industry. First, \textbf{A/B testing} is used across marketing and design \cite{kohavi2009controlled, box2005statistics} to study population traits. Here, the design is the choice of the A and B group sizes and the Bayesian model is a Gaussian linear model. Second, revealed \textbf{preference} \cite{samuelson1948consumption} is used in economics to understand consumer behaviour. We consider an experiment design setting in which we aim to learn the underlying utility function of an economic agent by presenting them with a proposal (such as offering them a price for a commodity) and observing their revealed preference.
Third, fixed effects and random effects (nuisance variables) are combined in \textbf{mixed effects} models~\cite{kruschke2014doing,gelman2013bayesian}. We consider an example inspired by item-response theory \cite{embretson2013item} in psychology. We seek information only about the fixed effects, making this an implicit likelihood problem.
Finally, we consider an experiment where labelled data from one region of design space must be used to predict labels in a target region by \textbf{extrapolation} \cite{mackay1992information}.  In summary, we have two models with
explicit likelihoods (A/B testing, preference) and two that are implicit (mixed effects, extrapolation). Full details of each model are presented in Appendix~\ref{sec:appendix:exp}.

For each scenario, we estimated the \acrshort{EIG} across a grid of designs with a fixed computational budget for each estimator and calculated the true \acrshort{EIG} analytically or with brute force computation as appropriate; see Table~\ref{tab:abserrors} for the results.
Whilst the Laplace method, unsurprisingly, performed best for the Gaussian linear model where its approximation becomes exact, we see that our methods are otherwise more accurate. All our methods outperformed \acrshort{NMC}.

\subsection{Convergence rates}

\begin{figure*}[t]
	\centering
	\begin{subfigure}[b]{0.23\textwidth}
		\centering
		\includegraphics[width=\textwidth]{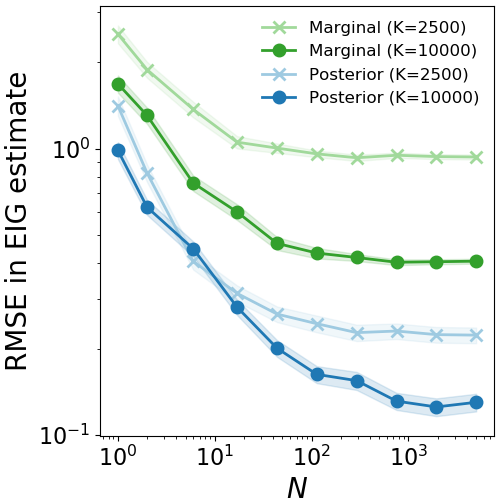}
		\caption{Convergence in $N$\label{fig:convN}}
	\end{subfigure}
	~
	\begin{subfigure}[b]{0.23\textwidth}
		\centering
		\includegraphics[width=\textwidth]{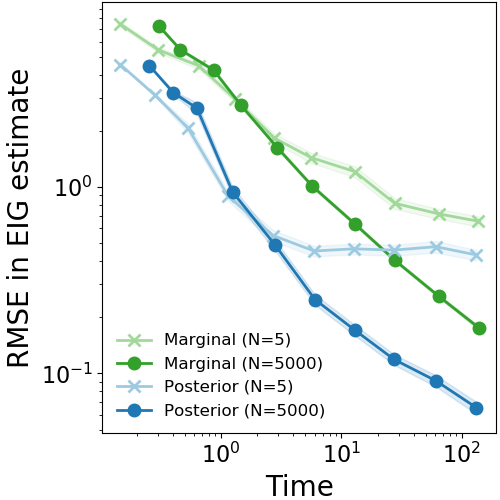}
		\caption{Convergence in $K$\label{fig:convK}}
	\end{subfigure}
	~
	\begin{subfigure}[b]{0.23\textwidth}
		\centering
		\includegraphics[width=\textwidth]{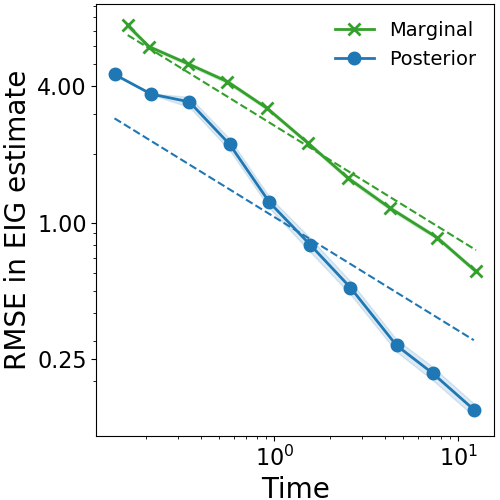}
		\caption{Convergence $N=K$\label{fig:convNK}}
	\end{subfigure}
	~
	\begin{subfigure}[b]{0.23\textwidth}
		\centering
		\includegraphics[width=\textwidth]{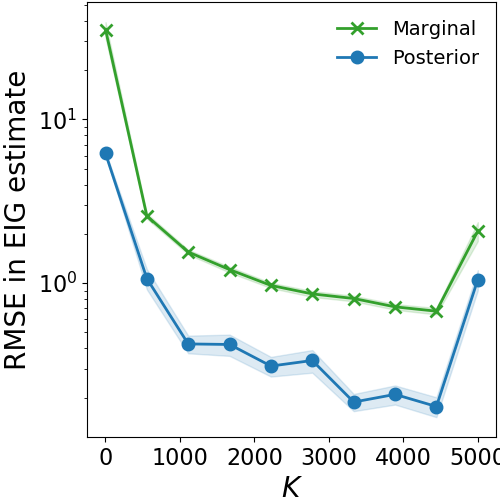}
		\caption{Fixed budget $N+K$\label{fig:convFixed}}
	\end{subfigure}
	\caption{Convergence of RMSE for $\ipost$ and $\imarg$.
		(a) Convergence in number of MC samples $N$ with a fixed number $K$ of gradient updates of the variational parameters. 
		(b) Convergence in time when increasing $K$ and with $N$ fixed.
		(c) Convergence in time when setting $N=K$ and increasing both (dashed lines represent theoretical rates).
		(d) Final RMSE with $N+K=5000$ fixed, for different $K$.
		Each graph shows the mean with shading representing $\pm1$ std.~err.~from 100 trials.}
		\vspace{-5pt}
	\label{fig:pmconvergence}
\end{figure*}

We now investigate the empirical convergence characteristics of our estimators. Throughout, we consider a single design point from the A/B test example. We start by examining the convergence of $\ipost$ and $\imarg$ as we allocate the computational budget in different ways. 

We first consider the convergence in $N$ after a fixed number of $K$ updates to the variational parameters.
As shown in Figure~\ref{fig:convN}, the RMSE initially decreases as we increase $N$, before plateauing due to the bias in the estimator.
We also see that $\ipost$ substantially outperforms $\imarg$.
We next consider the convergence as a function of wall-clock time when $N$ is held fixed and we increase $K$.
We see in Figure~\ref{fig:convK} that, as expected, the errors decrease with time and that when a small value of $N=5$ is taken, we again see a plateauing effect, with the variance of the final \acrshort{MC} estimator now becoming the limiting factor.
In Figure~\ref{fig:convNK} we take $N=K$ and increase both, obtaining the predicted convergence rate $\mathcal{O}(T^{-1/2})$ (shown by the dashed lines). 
We conjecture that the better performance of $\ipost$ is likely due to $\theta$ being lower dimensional (${\rm dim}=2$)
than $y$ (${\rm dim}=10$).
In Figure~\ref{fig:convFixed}, we instead fix $T=N+K$ to investigate the optimal trade-off between optimization and \acrshort{MC} error: it appears the range of $K/T$ between $0.5$ and $0.9$ gives the lowest RMSE. 

\begin{wrapfigure}{r}{0.4\columnwidth}
	\vspace{-10pt}
	\begin{center}
		\includegraphics[width=0.35\textwidth]{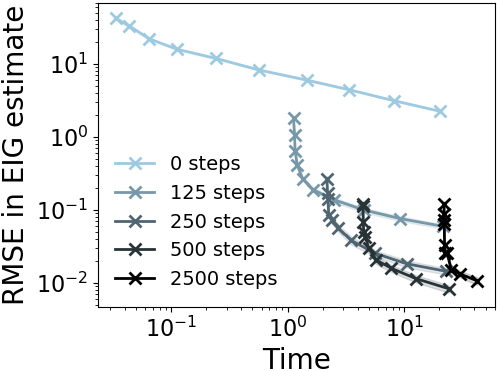}
		\caption{Convergence of $\ivnmc$ taking $M \!=\! \sqrt{N}$. `Steps' refers to pre-training of the variational posterior (i.e. $K$), with 0 steps corresponding to $\inmc$.  
			Means and confidence intervals as per Fig.~\ref{fig:pmconvergence}.
			\label{fig:vnmcconvergence}}
	\end{center}
	\vspace{-30pt}
\end{wrapfigure}

Finally, we show how $\ivnmc$ can improve over NMC by using an improved variational proposal for estimating $p(y|d)$. In Figure~\ref{fig:vnmcconvergence}, we plot the EIG estimates obtained by first running $K$ steps of stochastic gradient with $L=1$ to learn $q_v(\theta|y,d)$, before increasing $M$ and $N$. We see that spending some of our time budget training $q_v(\theta|y,d)$ leads to noticeable improvements in the estimation, but also that it is important to increase $N$ and $M$.
Rather than plateauing like $\ipost$ and $\imarg$, $\ivnmc$ continues to improve after the initial training period as, albeit at a slower $O(T^{-1/3})$ rate.

\subsection{End-to-end sequential experiments}
\label{sec:mturk}

We now demonstrate the utility of our methods for designing sequential experiments.
First, we demonstrate that our variational estimators are sufficiently robust and  fast to be used for adaptive experiments with a class of models that are of practical importance in many scientific disciplines. To this end, we run an adaptive psychology experiment with human participants recruited from Amazon Mechanical Turk to study how humans respond to features of stylized faces. To account for fixed effects---those \emph{common} across the population---as well as individual variations that we treat as nuisance variables, we use the mixed effects regression model introduced in Sec.~\ref{sec:eigaccuracy}. 
See Appendix~\ref{sec:appendix:exp} for full details of the experiment.

To estimate the EIG for different designs, we use $\iml$, 
since it yields the best performance on our mixed effects model benchmark (see Table~\ref{tab:abserrors}). 
Our \acrshort{EIG} estimator is integrated into a system that presents participants with a stimulus, receives their response, learns an updated model, and designs the next stimulus, all online.
Despite the relative simplicity of the design problem (with 36 possible designs) using BOED with $\iml$ leads to a more certain (i.e.~lower entropy) posterior than random design; 
see Figure~\ref{fig:turkexperiment}.

\begin{wrapfigure}{r}{0.45\columnwidth}
\vspace{-12pt}
	\begin{center}
		\includegraphics[width=.4\textwidth]{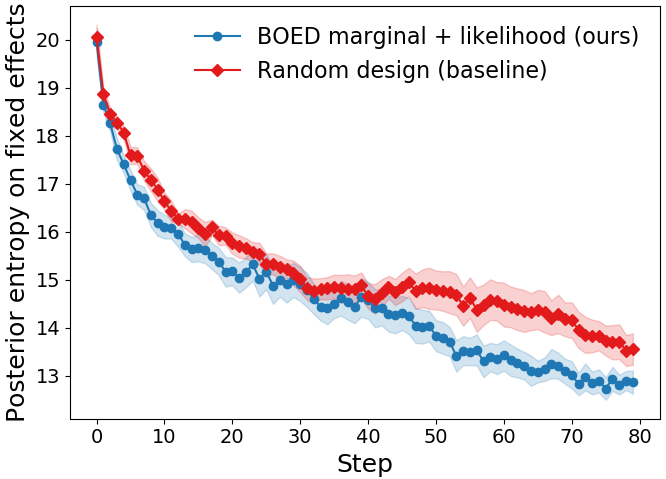}
		\caption{Evolution of the posterior entropy of the fixed effects in the Mechanical Turk experiment in Sec.~\ref{sec:mturk}. We depict the mean and $\pm1$ std.~err.~from 10 experimental trials.}
		\label{fig:turkexperiment}
	\end{center}
	\vspace{-12pt}
	
\end{wrapfigure}

Second, we consider a more challenging scenario in which a random design strategy gleans very little. We compare random design against two \acrshort{OED} strategies: $\imarg$ and $\inmc$. Building on the revealed preference example in Sec.~\ref{sec:eigaccuracy}, 
we consider an experiment to infer an agent's utility function which we model using the
Constant Elasticity of Substitution (CES) model \cite{arrow1961capital} with latent variables $\rho, \bm{\alpha}, u$.
We seek designs for which the agent's response will be informative about $\theta = (\rho,\bm{\alpha},u)$.
See Appendix~\ref{sec:appendix:exp} for full details. 
We estimate the \acrshort{EIG} using $\imarg$ because the dimension of $y$ is smaller than that of $\theta$, and select designs $d \in[0,100]^6$ using Bayesian optimization. To investigate parameter recovery we simulate agent responses from the model with fixed values of $\rho,\bm{\alpha},u$. Figure \ref{fig:ces} shows that using BOED with our marginal estimator reduces posterior entropy \textit{and} concentrates more quickly on the true parameter values than both baselines. Random design makes no inroads into the learning problem, while \acrshort{OED} based on NMC 
particularly struggles at the outset when $p(\theta|d_{1:t-1},y_{1:t-1})$, the prior at iteration $t$,
is high variance. Our method selects informative designs throughout.

\begin{figure*}[t]
	\centering
	\begin{subfigure}[b]{0.76\textwidth}
		\centering
		\includegraphics[width=\textwidth]{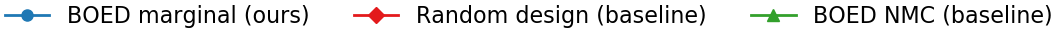}
	\end{subfigure}
	\hfill
	\begin{subfigure}[b]{0.24\textwidth}
		\centering
		\includegraphics[width=\textwidth]{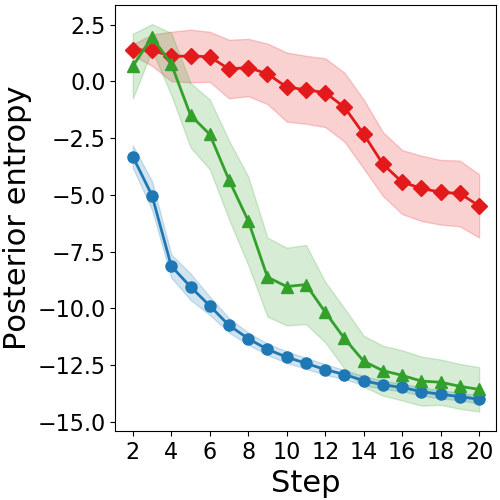}
		\caption{Entropy\label{fig:cesentropy}}
	\end{subfigure}
	\begin{subfigure}[b]{0.24\textwidth}
		\centering
		\includegraphics[width=\textwidth]{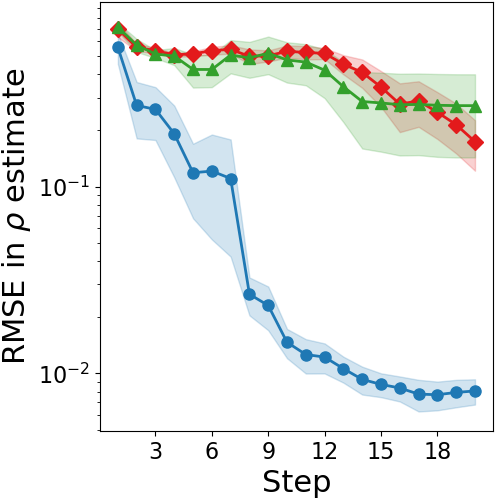}
		\caption{Posterior RMSE of $\rho$\label{fig:cesrho}}
	\end{subfigure}
	\begin{subfigure}[b]{0.24\textwidth}
		\centering
		\includegraphics[width=\textwidth]{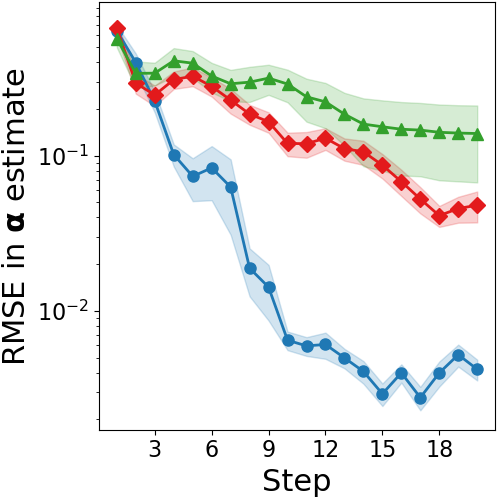}
		\caption{Posterior RMSE of $\bm{\alpha}$ \label{fig:cesalpha}}
	\end{subfigure}
	\begin{subfigure}[b]{0.24\textwidth}
		\centering
		\includegraphics[width=\textwidth]{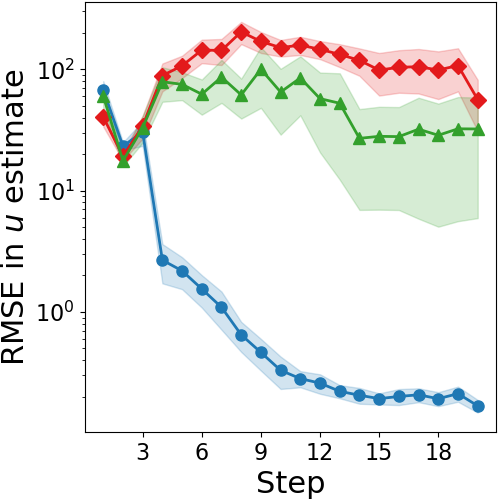}
		\caption{Posterior RMSE of $u$ \label{fig:cesu}}
	\end{subfigure}
	\caption{Evolution of the posterior in the sequential CES experiment. (a) Total entropy of a mean-field variational approximation of the posterior. (b)(c)(d) The RMSE of the posterior approximations of $\rho$, $\bm{\alpha}$ and $u$ as compared to the true values used to simulate agent responses. Note the scale of the vertical axis is logarithmic. All plots show the mean and $\pm1$ std.~err.~from 10 independent runs.}
	\label{fig:ces}
\end{figure*}



\section{Discussion}
\label{sec:discussion}

We have developed efficient \acrshort{EIG} estimators that are applicable to a wide range of experimental design problems.
By tackling the double intractability of the \acrshort{EIG} in a principled manner, they provide substantially
improved convergence rates relative to previous approaches, and our experiments show that these theoretical
advantages translate into significant practical gains.
Our estimators are well-suited to modern deep probabilistic programming languages and we have provided an implementation
in Pyro.
We note that the interplay between variational and MC methods in EIG estimation is not directly analogous
to those in standard inference settings because the NMC EIG estimator is itself inherently biased.
Our $\ivnmc$ estimator allows one to play off the advantages of these approaches, namely the fast learning of 
variational approaches and asymptotic consistency of NMC.


\section*{Acknowledgements}

We gratefully acknowledge research funding from Uber AI Labs.
MJ would like to thank Paul Szerlip for help generating the sprites used in the Mechanical Turk experiment.
AF would like to thank Patrick Rebeschini, Dominic Richards and Emile Mathieu for their help and support.
AF gratefully acknowledges funding from EPSRC grant no. EP/N509711/1.
YWT's and TR's research leading to these results has received funding from the
European Research Council under the European Union's Seventh Framework
Programme (FP7/2007-2013) ERC grant agreement no. 617071.

\bibliography{references}

\begin{thebibliography}{43}
\providecommand{\natexlab}[1]{#1}
\providecommand{\url}[1]{\texttt{#1}}
\expandafter\ifx\csname urlstyle\endcsname\relax
  \providecommand{\doi}[1]{doi: #1}\else
  \providecommand{\doi}{doi: \begingroup \urlstyle{rm}\Url}\fi

\bibitem[Amzal et~al.(2006)Amzal, Bois, Parent, and Robert]{amzal2006bayesian}
Billy Amzal, Fr{\'e}d{\'e}ric~Y Bois, Eric Parent, and Christian~P Robert.
\newblock Bayesian-optimal design via interacting particle systems.
\newblock \emph{Journal of the American Statistical association}, 101\penalty0
  (474):\penalty0 773--785, 2006.

\bibitem[Arrow et~al.(1961)Arrow, Chenery, Minhas, and Solow]{arrow1961capital}
Kenneth~J Arrow, Hollis~B Chenery, Bagicha~S Minhas, and Robert~M Solow.
\newblock Capital-labor substitution and economic efficiency.
\newblock \emph{The review of Economics and Statistics}, pages 225--250, 1961.

\bibitem[Barber and Agakov(2003)]{ba}
David Barber and Felix Agakov.
\newblock The {IM} algorithm: a variational approach to information
  maximization.
\newblock \emph{Advances in Neural Information Processing Systems},
  16:\penalty0 201--208, 2003.

\bibitem[Belghazi et~al.(2018)Belghazi, Rajeswar, Baratin, Hjelm, and
  Courville]{mine}
Ishmael Belghazi, Sai Rajeswar, Aristide Baratin, R~Devon Hjelm, and Aaron
  Courville.
\newblock {MINE}: mutual information neural estimation.
\newblock \emph{arXiv preprint arXiv:1801.04062}, 2018.

\bibitem[Bingham et~al.(2019)Bingham, Chen, Jankowiak, Obermeyer, Pradhan,
  Karaletsos, Singh, Szerlip, Horsfall, and Goodman]{pyro}
Eli Bingham, Jonathan~P Chen, Martin Jankowiak, Fritz Obermeyer, Neeraj
  Pradhan, Theofanis Karaletsos, Rohit Singh, Paul Szerlip, Paul Horsfall, and
  Noah~D Goodman.
\newblock Pyro: Deep universal probabilistic programming.
\newblock \emph{The Journal of Machine Learning Research}, 20\penalty0
  (1):\penalty0 973--978, 2019.

\bibitem[Box et~al.(2005)Box, Hunter, and Hunter]{box2005statistics}
George~EP Box, J~Stuart Hunter, and William~G Hunter.
\newblock Statistics for experimenters.
\newblock In \emph{Wiley Series in Probability and Statistics}. Wiley Hoboken,
  NJ, 2005.

\bibitem[Burda et~al.(2015)Burda, Grosse, and
  Salakhutdinov]{burda2015importance}
Yuri Burda, Roger Grosse, and Ruslan Salakhutdinov.
\newblock Importance weighted autoencoders.
\newblock \emph{arXiv preprint arXiv:1509.00519}, 2015.

\bibitem[Chaloner and Verdinelli(1995)]{chaloner1995}
Kathryn Chaloner and Isabella Verdinelli.
\newblock Bayesian experimental design: A review.
\newblock \emph{Statistical Science}, pages 273--304, 1995.

\bibitem[Cook et~al.(2008)Cook, Gibson, and Gilligan]{cook2008optimal}
Alex~R Cook, Gavin~J Gibson, and Christopher~A Gilligan.
\newblock Optimal observation times in experimental epidemic processes.
\newblock \emph{Biometrics}, 64\penalty0 (3):\penalty0 860--868, 2008.

\bibitem[Dayan et~al.(1995)Dayan, Hinton, Neal, and Zemel]{dayan1995helmholtz}
Peter Dayan, Geoffrey~E Hinton, Radford~M Neal, and Richard~S Zemel.
\newblock The {H}elmholtz machine.
\newblock \emph{Neural computation}, 7\penalty0 (5):\penalty0 889--904, 1995.

\bibitem[Donsker and Varadhan(1975)]{donsker1975asymptotic}
Monroe~D Donsker and SR~Srinivasa Varadhan.
\newblock Asymptotic evaluation of certain {M}arkov process expectations for
  large time.
\newblock \emph{Communications on Pure and Applied Mathematics}, 28\penalty0
  (1):\penalty0 1--47, 1975.

\bibitem[Ehrenfeld(1962)]{ehrenfeld1962some}
Sylvain Ehrenfeld.
\newblock Some experimental design problems in attribute life testing.
\newblock \emph{Journal of the American Statistical Association}, 57\penalty0
  (299):\penalty0 668--679, 1962.

\bibitem[Embretson and Reise(2013)]{embretson2013item}
Susan~E Embretson and Steven~P Reise.
\newblock \emph{Item response theory}.
\newblock Psychology Press, 2013.

\bibitem[Gelman et~al.(2013)Gelman, Stern, Carlin, Dunson, Vehtari, and
  Rubin]{gelman2013bayesian}
Andrew Gelman, Hal~S Stern, John~B Carlin, David~B Dunson, Aki Vehtari, and
  Donald~B Rubin.
\newblock \emph{Bayesian data analysis}.
\newblock Chapman and Hall/CRC, 2013.

\bibitem[Golovin et~al.(2010)Golovin, Krause, and Ray]{golovin2010}
Daniel Golovin, Andreas Krause, and Debajyoti Ray.
\newblock Near-optimal bayesian active learning with noisy observations.
\newblock In \emph{Advances in Neural Information Processing Systems}, pages
  766--774, 2010.

\bibitem[Hern{\'a}ndez-Lobato et~al.(2014)Hern{\'a}ndez-Lobato, Hoffman, and
  Ghahramani]{hernandez2014}
Jos{\'e}~Miguel Hern{\'a}ndez-Lobato, Matthew~W Hoffman, and Zoubin Ghahramani.
\newblock Predictive entropy search for efficient global optimization of
  black-box functions.
\newblock In \emph{Advances in neural information processing systems}, pages
  918--926, 2014.

\bibitem[Kingma and Welling(2014)]{kingma2014auto}
Diederik~P Kingma and Max Welling.
\newblock Auto-encoding variational {Bayes}.
\newblock In \emph{ICLR}, 2014.

\bibitem[Kleinegesse and Gutmann(2018)]{kleinegesse2018efficient}
Steven Kleinegesse and Michael Gutmann.
\newblock Efficient {B}ayesian experimental design for implicit models.
\newblock \emph{arXiv preprint arXiv:1810.09912}, 2018.

\bibitem[Kohavi et~al.(2009)Kohavi, Longbotham, Sommerfield, and
  Henne]{kohavi2009controlled}
Ron Kohavi, Roger Longbotham, Dan Sommerfield, and Randal~M Henne.
\newblock Controlled experiments on the web: survey and practical guide.
\newblock \emph{Data mining and knowledge discovery}, 18\penalty0 (1):\penalty0
  140--181, 2009.

\bibitem[Kruschke(2014)]{kruschke2014doing}
John Kruschke.
\newblock \emph{Doing Bayesian data analysis: A tutorial with R, JAGS, and
  Stan}.
\newblock Academic Press, 2014.

\bibitem[Le et~al.(2017)Le, Igl, Rainforth, Jin, and Wood]{le2017auto}
Tuan~Anh Le, Maximilian Igl, Tom Rainforth, Tom Jin, and Frank Wood.
\newblock Auto-encoding sequential monte carlo.
\newblock \emph{arXiv preprint arXiv:1705.10306}, 2017.

\bibitem[Lewi et~al.(2009)Lewi, Butera, and Paninski]{lewi2009sequential}
Jeremy Lewi, Robert Butera, and Liam Paninski.
\newblock Sequential optimal design of neurophysiology experiments.
\newblock \emph{Neural Computation}, 21\penalty0 (3):\penalty0 619--687, 2009.

\bibitem[Lindley(1956)]{lindley1956}
Dennis~V Lindley.
\newblock On a measure of the information provided by an experiment.
\newblock \emph{The Annals of Mathematical Statistics}, pages 986--1005, 1956.

\bibitem[Lindley(1972)]{lindley1972}
Dennis~V Lindley.
\newblock \emph{Bayesian statistics, a review}, volume~2.
\newblock SIAM, 1972.

\bibitem[Long et~al.(2013)Long, Scavino, Tempone, and Wang]{long2013}
Quan Long, Marco Scavino, Ra{\'u}l Tempone, and Suojin Wang.
\newblock Fast estimation of expected information gains for {B}ayesian
  experimental designs based on {L}aplace approximations.
\newblock \emph{Computer Methods in Applied Mechanics and Engineering},
  259:\penalty0 24--39, 2013.

\bibitem[Ma et~al.(2018)Ma, Tschiatschek, Palla, Lobato, Nowozin, and
  Zhang]{ma2018eddi}
Chao Ma, Sebastian Tschiatschek, Konstantina Palla, Jose Miguel~Hernandez
  Lobato, Sebastian Nowozin, and Cheng Zhang.
\newblock {EDDI}: Efficient dynamic discovery of high-value information with
  partial {VAE}.
\newblock \emph{arXiv preprint arXiv:1809.11142}, 2018.

\bibitem[MacKay(1992)]{mackay1992information}
David~JC MacKay.
\newblock Information-based objective functions for active data selection.
\newblock \emph{Neural computation}, 4\penalty0 (4):\penalty0 590--604, 1992.

\bibitem[Moulines and Bach(2011)]{moulines2011}
Eric Moulines and Francis~R Bach.
\newblock Non-asymptotic analysis of stochastic approximation algorithms for
  machine learning.
\newblock In \emph{Advances in Neural Information Processing Systems}, pages
  451--459, 2011.

\bibitem[M{\"u}ller(2005)]{muller2005simulation}
Peter M{\"u}ller.
\newblock Simulation based optimal design.
\newblock \emph{Handbook of Statistics}, 25:\penalty0 509--518, 2005.

\bibitem[Myung et~al.(2013)Myung, Cavagnaro, and Pitt]{myung2013}
Jay~I Myung, Daniel~R Cavagnaro, and Mark~A Pitt.
\newblock A tutorial on adaptive design optimization.
\newblock \emph{Journal of mathematical psychology}, 57\penalty0
  (3-4):\penalty0 53--67, 2013.

\bibitem[Poole et~al.(2018)Poole, Ozair, van~den Oord, Alemi, and
  Tucker]{poole2018variational}
Ben Poole, Sherjil Ozair, A{\"a}ron van~den Oord, Alexander~A Alemi, and George
  Tucker.
\newblock On variational lower bounds of mutual information.
\newblock \emph{NeurIPS Workshop on Bayesian Deep Learning}, 2018.

\bibitem[Rainforth(2017)]{rainforth2017thesis}
Tom Rainforth.
\newblock \emph{Automating Inference, Learning, and Design using Probabilistic
  Programming}.
\newblock PhD thesis, University of Oxford, 2017.

\bibitem[Rainforth et~al.(2018)Rainforth, Cornish, Yang, Warrington, and
  Wood]{nmc}
Tom Rainforth, Robert Cornish, Hongseok Yang, Andrew Warrington, and Frank
  Wood.
\newblock On nesting {Monte Carlo} estimators.
\newblock In \emph{International Conference on Machine Learning}, pages
  4264--4273, 2018.

\bibitem[Rezende et~al.(2014)Rezende, Mohamed, and
  Wierstra]{rezende2014stochastic}
Danilo~Jimenez Rezende, Shakir Mohamed, and Daan Wierstra.
\newblock Stochastic backpropagation and approximate inference in deep
  generative models.
\newblock In \emph{Proceedings of the 31st International Conference on Machine
  Learning}, volume~32, pages 1278--1286, 2014.

\bibitem[Robbins and Monro(1951)]{robbins1951stochastic}
Herbert Robbins and Sutton Monro.
\newblock A stochastic approximation method.
\newblock \emph{The annals of mathematical statistics}, pages 400--407, 1951.

\bibitem[Samuelson(1948)]{samuelson1948consumption}
Paul~A Samuelson.
\newblock Consumption theory in terms of revealed preference.
\newblock \emph{Economica}, 15\penalty0 (60):\penalty0 243--253, 1948.

\bibitem[Sebastiani and Wynn(2000)]{sebastiani2000maximum}
Paola Sebastiani and Henry~P Wynn.
\newblock Maximum entropy sampling and optimal {B}ayesian experimental design.
\newblock \emph{Journal of the Royal Statistical Society: Series B (Statistical
  Methodology)}, 62\penalty0 (1), 2000.

\bibitem[Shababo et~al.(2013)Shababo, Paige, Pakman, and
  Paninski]{shababo2013bayesian}
Ben Shababo, Brooks Paige, Ari Pakman, and Liam Paninski.
\newblock Bayesian inference and online experimental design for mapping neural
  microcircuits.
\newblock In \emph{Advances in Neural Information Processing Systems}, pages
  1304--1312, 2013.

\bibitem[Snoek et~al.(2012)Snoek, Larochelle, and Adams]{snoek2012practical}
Jasper Snoek, Hugo Larochelle, and Ryan~P Adams.
\newblock Practical {B}ayesian optimization of machine learning algorithms.
\newblock In \emph{Advances in neural information processing systems}, pages
  2951--2959, 2012.

\bibitem[Stuhlm{\"u}ller et~al.(2013)Stuhlm{\"u}ller, Taylor, and
  Goodman]{stuhlmuller2013learning}
Andreas Stuhlm{\"u}ller, Jacob Taylor, and Noah Goodman.
\newblock Learning stochastic inverses.
\newblock In \emph{Advances in neural information processing systems}, pages
  3048--3056, 2013.

\bibitem[Thomas et~al.(2016)Thomas, Dutta, Corander, Kaski, and
  Gutmann]{dutta2016likelihood}
Owen Thomas, Ritabrata Dutta, Jukka Corander, Samuel Kaski, and Michael~U
  Gutmann.
\newblock Likelihood-free inference by ratio estimation.
\newblock \emph{arXiv preprint arXiv:1611.10242}, 2016.

\bibitem[Vanlier et~al.(2012)Vanlier, Tiemann, Hilbers, and van
  Riel]{vanlier2012}
Joep Vanlier, Christian~A Tiemann, Peter~AJ Hilbers, and Natal~AW van Riel.
\newblock A {B}ayesian approach to targeted experiment design.
\newblock \emph{Bioinformatics}, 28\penalty0 (8):\penalty0 1136--1142, 2012.

\bibitem[Vincent and Rainforth(2017)]{vincent2017}
Benjamin~T Vincent and Tom Rainforth.
\newblock The {DARC} toolbox: automated, flexible, and efficient delayed and
  risky choice experiments using bayesian adaptive design.
\newblock 2017.

\end{thebibliography}
\bibliographystyle{plainnat}

\clearpage

\appendix

\section{Details for variational estimators}
\label{sec:app:method}

The proofs in~\ref{sec:app:ipost} and~\ref{sec:app:imarg} are included for completeness.

\subsection{Variational posterior $\ipost$}
\label{sec:app:ipost}

We require valid approximations $\qp(\theta|y,d)$ to have the same support as $p(\theta|y,d)$. Recall
\begin{equation}
	\lowerbound(d) = \mathbb{E}_{p(y, \theta | d)}\left[ \log \frac{\qp(\theta | y, d)}{p(\theta)} \right]
\end{equation}
and
\begin{equation}
	\eig(d) = \mathbb{E}_{p(y, \theta | d)}\left[ \log \frac{p(\theta | y, d)}{p(\theta)} \right]
\end{equation}

We aim to show $\eig(d) \ge \lowerbound(d)$. Following \cite{ba}, we have
\begin{align}
	\eig(d) - \lowerbound(d)
	 = & \mathbb{E}_{p(y, \theta | d)}\left[ \log \frac{p(\theta | y, d)}{p(\theta)} - \log \frac{\qp(\theta|y,d)}{p(\theta)}\right] \\
	= & \mathbb{E}_{p(y, \theta | d)}\left[ \log \frac{p(\theta | y, d)p(\theta)}{p(\theta)\qp(\theta|y,d)} \right] \\
	 = & \mathbb{E}_{p(y|d)} \left[\mathbb{E}_{p(\theta|y,d)} \left[\log \frac{p(\theta | y,d)}{\qp(\theta | y, d)} \right]\right] \\
	 = & \mathbb{E}_{p(y|d)} \left[\text{KL}\left(p(\theta | y,d)||\qp(\theta | y, d)\right)\right] \\
	 \ge & 0.
\end{align}
To further prove that the bound is tight, we note that the penultimate term $\mathbb{E}_{p(y|d)} \left[\text{KL}\left(p(\theta | y,d)||\qp(\theta | y, d)\right)\right]$ equals $0$ if and only if  $\text{KL}\left(p(\theta | y,d)||\qp(\theta | y, d)\right) = 0$ for almost all $y$ (i.e. the union of all $y$ for which this does not hold has measure zero). The occurs if and only if $\qp(\theta|y,d) = p(\theta|y,d)$ for almost all $y,\theta$.

\subsection{Variational marginal $\imarg$}
\label{sec:app:imarg}
We now demonstrate that $\upperbound(d)$ is an upper bound on $\eig(d)$. Proceeding in the same manner as for $\ipost$, we find
\begin{align}
	\upperbound(d) - \eig(d) = & \mathbb{E}_{p(y, \theta | d)}\left[ \log \frac{p(y | \theta , d)}{\qm(y|d)} - \log \frac{p(y|\theta,d)}{p(y|d)}\right] \\
	= & \mathbb{E}_{p(y, \theta | d)}\left[ \log \frac{p(y|\theta, d)p(y|d)}{\qm(y|d)p(y|\theta,d)} \right] \\
	 = & \mathbb{E}_{p(y|d)} \left[\log \frac{p(y|d)}{\qm(y| d)} \right] \\
	 = & \text{KL}\left(p(y|d)||\qm(y| d)\right) \\
	 \ge & 0.
\end{align}
Again, the bound is tight if and only if $\qm(y|d) = p(y|d)$ almost everywhere.

\subsection{Variational NMC $\ivnmc$}
We now prove Lemma~\ref{lemma:vnmc} from the main paper, duplicating the Lemma itself below for convenience.
\lemvnmc*
\begin{proof}
Starting with proving the first result in lemma, we first recall the definition of $\upperboundvnmc(d,L)$ itself,
\begin{equation}
\upperboundvnmc(d,L) = \expect\left[ \log p(y|\theta_0,d) - \log \frac{1}{L} \sum_{\ell=1}^L \frac{p(y,\theta_\ell|d)}{q_v(\theta_\ell|y,d)} \right]
\end{equation}
where the expectation is taken over $y,\theta_{0:L} \sim p(y,\theta_0|d)\prod_{\ell=1}^L q_v(\theta_\ell|y,d)$.
We consider positive integers $L_2 \ge L_1$. We let $\delta = \upperboundvnmc(d, L_1) - \upperboundvnmc(d, L_2)$. Then,
\begin{equation}
	\delta = \mathbb{E}\left[  \log \frac{1}{L_2}\sum_{\ell=1}^{L_2} \frac{p(y,\theta_\ell|d)}{q_v(\theta_\ell|y,d)} \right]- \mathbb{E}\left[  \log \frac{1}{L_1} \sum_{\ell=1}^{L_1} \frac{p(y,\theta_\ell|d)}{q_v(\theta_\ell|y,d)} \right].
\end{equation}
We now proceed as in \cite{burda2015importance}.  Let $I_1, ..., I_{L_1}$ be distinct indices drawn uniformly from $1, ..., L_2$. Then,
\begin{equation}
	\frac{1}{L_2}\sum_{\ell=1}^{L_2} \frac{p(y,\theta_\ell)}{q_v(\theta_\ell|y,d)} = \mathbb{E}_{I_1, ..., I_{L_1}} \left[ \frac{1}{L_1} \sum_{j=1}^{L_1} \frac{p(y,\theta_{I_j})}{q_v(\theta_{I_j}|y,d)} \right]
\end{equation}
So
\begin{equation}
	\delta = \mathbb{E}\left[  \log \left(\mathbb{E}_{I_{1:L_1}} \left[ \frac{1}{L_1} \sum_{j=1}^{L_1} \frac{p(y,\theta_{I_j})}{q_v(\theta_{I_j}|y,d)} \right] \right)\right]
	- \mathbb{E}\left[ \log \frac{1}{L_1}\sum_{\ell=1}^{L_1} \frac{p(y,\theta_\ell|d)}{q_v(\theta_\ell|y,d)} \right],
\end{equation}
then by Jensen's Inequality
\begin{align}
	\delta &\ge \mathbb{E}\left[  \mathbb{E}_{I_{1:L_1}} \left[ \log \left(\frac{1}{L_1} \sum_{j=1}^{L_1} \frac{p(y,\theta_{I_j})}{q_v(\theta_{I_j}|y,d)}\right) \right] \right] - 
	\mathbb{E}\left[ \log \frac{1}{L_1}\sum_{\ell=1}^{L_1} \frac{p(y,\theta_\ell|d)}{q_v(\theta_\ell|y,d)} \right] \\
	&\ge \mathbb{E}\left[ \log \frac{1}{L_1}\sum_{\ell=1}^{L_1} \frac{p(y,\theta_\ell|d)}{q_v(\theta_\ell|y,d)} \right] - \mathbb{E}\left[ \log \frac{1}{L_1}\sum_{\ell=1}^{L_1} \frac{p(y,\theta_\ell|d)}{q_v(\theta_\ell|y,d)} \right] \\
	&\ge 0
\end{align}
where we have used that $\theta_{I_1}, ..., \theta_{I_{L_1}} \overset{d}{=} \theta_1, ..., \theta_{L_1}$. This shows that $\upperboundvnmc(d, L_1) \ge \upperboundvnmc(d, L_2)$. For the limit $\lim_{L\to\infty} \upperboundvnmc(d, L)$ we first fix some $y$ for which $p(y|d)>0$ and consider
\begin{equation}
	\upperboundvnmc(d, L, y) = \mathbb{E}\left[ \log p(y|\theta_0,d) - \log \frac{1}{L}\sum_{\ell=1}^L \frac{p(y,\theta_\ell|d)}{q_v(\theta_\ell|y,d)} \right].
\end{equation}
with the expectation taken over $p(\theta_0|y,d)\prod_{\ell=1}^L q_v(\theta_\ell|y,d)$. Since $p(y,\theta|d)/q_v(\theta|y,d)$ is bounded by assumption, the Strong Law of Large Numbers implies that, in limit of large $L$,
\begin{equation}
	\frac{1}{L}\sum_{\ell=1}^L \frac{p(y,\theta_\ell|d)}{q_v(\theta_\ell|y,d)} \to p(y|d) \; a.s.
\end{equation}
Furthermore, using the same argument as before, $\upperboundvnmc(d, L_1, y) \ge \upperboundvnmc(d, L_2, y)$ whenever $L_2 \ge L_1$. 
Thus the Bounded Convergence Theorem implies
\begin{equation}
	\upperboundvnmc(d, L, y) \downarrow \mathbb{E}_{p(\theta_0|y,d)}[\log p(y|\theta_0,d) - \log p(y|d)] \text{ as } L \to \infty
\end{equation}
so, taking expectations of $p(y|d)$, by the Monotone Convergence Theorem
\begin{equation}
	\upperboundvnmc(d, L) \downarrow \mathbb{E}_{p(y,\theta_0|d)}[\log p(y|\theta_0,d) - \log p(y|d)] = \eig(d) \text{ as } L \to \infty.
\end{equation}

For the second result, we simply note that
\begin{equation}
	\frac{p(y,\theta|d)}{p(\theta|y,d)} = \frac{p(y,\theta|d)}{\frac{p(y,\theta|d)}{p(y|d)}} = p(y|d)
\end{equation}

Finally, for the third result, we proceed as in \cite{le2017auto}. We have
\begin{align}
	\upperboundvnmc(d, L) - \eig(d) &= \mathbb{E}\left[\log p(y|d) - \log \frac{1}{L}\sum_{\ell=1}^l \frac{p(y,\theta_\ell|d)}{q_v(\theta_\ell|y,d)} \right]
\end{align}
where the expectation is over $p(y,\theta_0|d)\prod_{\ell=1}^L q_v(\theta_\ell|y,d)$.

Then
\begin{align}
	\upperboundvnmc(d, L) - \eig(d) &= \mathbb{E}\left[ - \log\frac{1}{L}\sum_{\ell=1}^L \frac{p(\theta_\ell|y,d)}{q_v(\theta_\ell|y,d)} \right] \\
	&= \mathbb{E}\left[\log \frac{\prod_{\ell=1}^{L}q_v(\theta_\ell|y,d)}{\frac{1}{L} \sum_{\ell=1}^L p(\theta_\ell|y,d)\prod_{k\ne\ell} q_v(\theta_k|y,d)} \right] \\
	&= \mathbb{E}\left[\log \frac{\prod_{\ell=1}^L q_v(\theta_\ell|y,d)}{P(\theta_{1:L}|y,d)} \right] \\
	&= \mathbb{E}_{p(y|d)}\left[\text{KL}\left(\prod_{\ell=1}^L q_v(\theta_\ell|y,d) || P(\theta_{1:L}|y,d) \right) \right]
\end{align}
where $P(\theta_{1:L}|y,d) = \frac{1}{L}\sum_{\ell=1}^L p(\theta_\ell|y,d) \prod_{k\ne\ell} q_v(\theta_k|y,d)$.
\end{proof}

\subsection{Variational marginal + likelihood $\iml$}
\label{sec:app:iml}
We now prove Lemma~\ref{lemma:ml} from the main paper, duplicating the Lemma itself below for convenience.
\lemiml*
\begin{proof}
	We aim to bound $|\Iml(d) - \eig(d)|$. Let $\delta = \Iml(d) - \eig(d)$. We have
	\begin{align}
		\delta &= \mathbb{E}_{p(y,\theta|d)}\left[\log \frac{\ql(y|\theta,d)}{\qm(y|d)} \right] - \mathbb{E}_{p(y,\theta|d)}\left[\log \frac{p(y|\theta,d)}{p(y|d)} \right] \\
		&= \mathbb{E}_{p(y,\theta|d)}\left[\log \frac{\ql(y|\theta,d)}{\qm(y|d)} - \log \frac{p(y|\theta,d)}{p(y|d)} \right] \\
		&= \mathbb{E}_{p(y,\theta|d)}\left[\log \frac{\ql(y|\theta,d)}{\qm(y|d)} - \log \frac{p(y|\theta,d)}{\qm(y|d)} + \log \frac{p(y|\theta,d)}{\qm(y|d)} - \log \frac{p(y|\theta,d)}{p(y|d)} \right] \\
		&= -\mathbb{E}_{p(y,\theta|d)}\left[\log \frac{\qm(y|d)p(y|\theta,d)}{\ql(y|\theta,d)\qm(y|d)}\right] + \mathbb{E}_{p(y,\theta|d)}\left[\log \frac{p(y|\theta,d)p(y|d)}{\qm(y|d)p(y|\theta,d)}\right] \\
		&= -\mathbb{E}_{p(\theta)}\left[\mathbb{E}_{p(y|\theta,d)}\left[\log \frac{p(y|\theta,d)}{\ql(y|\theta,d)}\right]\right] + \mathbb{E}_{p(y|d)}\left[\log \frac{p(y|d)}{\qm(y|d)}\right] \\
		&= -\mathbb{E}_{p(\theta)}\left[\text{KL}(p(y|\theta,d)||\ql(y|\theta,d))\right] + \text{KL}(p(y|d)||\qm(y|d)).
	\end{align}
	So, by the triangle inequality
	\begin{align}
		|\delta| &\le \mathbb{E}_{p(\theta)}\left[\text{KL}(p(y|\theta,d)||\ql(y|\theta,d))\right] + \text{KL}(p(y|d)||\qm(y|d)).
		\label{eq:mlklbound}
	\end{align}
	We can rewrite the RHS using the following relation
	\begin{align}
		\text{KL}(p(x)||q(x)) &= \mathbb{E}_{p(x)}\left[ \log\frac{p(x)}{q(x)} \right] \\
		&= \mathbb{E}_{p(x)}[\log p(x)] - \mathbb{E}_{p(x)}[\log q(x)] \\
		&= -H[p(x)]- \mathbb{E}_{p(x)}[\log q(x)].
	\end{align}
	This gives us
	\begin{align}
		|\delta| &\le \mathbb{E}_{p(\theta)}\left[ -H(p(y|\theta,d) \right] - \mathbb{E}_{p(y,\theta|d)}[\log \ql(y|\theta,d)] - H[p(y|d)] - \mathbb{E}_{p(y,|d)}[\log \qm(y|d)] \\
		&\le  - \mathbb{E}_{p(y,\theta|d)}[\log \qm(y|d) + \log\ql(y|\theta,d)] -H[p(y|d)] - \mathbb{E}_{p(\theta)}\left[ H(p(y|\theta,d) \right]
	\end{align}
	as required.
	
	Finally, from \eqref{eq:mlklbound} we see that the error bound is tight if and only if both KL-divergences are 0 if and only if $\ql(y|\theta,d) = p(y|\theta,d)$ and $\qm(y|d)=p(y|d)$ for almost all $y,\theta$.
\end{proof}

\label{app:qmql}
We conclude with an additional observation. Suppose that we set $\qm(y|d) = \mathbb{E}_{p(\theta)}[\ql(y|\theta,d)]$. This could be possible for instance when $\theta$ takes finitely many values. In this case, $\Iml(d)$ is actually a lower bound on $\eig(d)$. This is in contrast to the general case when $\qm$ and $\ql$ are learned separately, in which it is neither an upper nor a lower bound.

To show that $\Iml(d)$ is a lower bound when $\qm(y|d) = \mathbb{E}_{p(\theta)}[\ql(y|\theta,d)]$, we begin with the Donsker-Varadhan bound \cite{donsker1975asymptotic}
\begin{equation}
	\eig(d) \ge \mathbb{E}_{p(y,\theta|d)}[T(y,\theta)] - \log \left(\mathbb{E}_{p(\theta)p(y|d)}[e^{T(y,\theta)}] \right).
\end{equation}
Substituting $T(y,\theta) = \log(\ql(y|\theta,d)/\qm(y|d))$ we have
\begin{align}
	\eig(d) &\ge \mathbb{E}_{p(y,\theta|d)}\left[\log\frac{\ql(y|\theta,d)}{\qm(y|d)}\right] - \log \left(\mathbb{E}_{p(\theta)p(y|d)}\left[\frac{\ql(y|\theta,d)}{\qm(y|d)}\right] \right) \\
	&\ge \Iml(d) - \log \left(\mathbb{E}_{p(y|d)}\left[\mathbb{E}_{p(\theta)}\left\{\frac{\ql(y|\theta,d)}{\qm(y|d)}\right\}\right] \right) \\
	&\ge \Iml(d) - \log \left(\mathbb{E}_{p(y|d)}\left[\frac{\mathbb{E}_{p(\theta)}\left\{\ql(y|\theta,d)\right\}}{\qm(y|d)}\right] \right) \\
	&\ge \Iml(d) - \log \left(\mathbb{E}_{p(y|d)}\left[\frac{\qm(y|d)}{\qm(y|d)}\right] \right) \\
	&\ge \Iml(d).
\end{align}



\section{Details for convergence rates}
\label{sec:app:theory}

We now provide the details for Theorem~\ref{the:main}.
Key to proving the aspect of the Theorem relating to the convergence of the variational parameter
$\phi_K$ to $\phi^*$ is Assumption~\ref{ass:sgd}.
Points 1-5 correspond to assumptions H2', H3, H4, H6, and H7 of~\cite{moulines2011}; our proof will rely heavily on theirs.  
We note that also that our measurability assumption made in the Theorem itself means that their assumption H1 is automatically satisfied.

\begin{assumption}
	Assume:
	\begin{enumerate}
		\item The function $\phi \mapsto f(X, \phi)$ is almost surely convex in its second argument and differentiable with Lipschitz continuous gradient, i.e. $\forall \phi_1, \phi_2 \in \Phi$:
		\begin{equation*}
			 \expect(\|\nabla f(X, \phi_1) - \nabla f(X, \phi_2)\|^2 ) \le C \|\phi_1 - \phi_2\|
		\end{equation*}
		with probability 1 for some $C$.
		\item The function $f$ is $\nu$-strongly convex; that is, for all $\phi_1, \phi_2 \in \Phi$:
		\begin{align*}
			f(X, \phi_1) \ge f(X, \phi_2) &+ \nabla f(X, \phi_2)^T (\phi_1 - \phi_2) \\
			&+ \tfrac{\nu}{2}\|\phi_1 - \phi_2\|^2
		\end{align*}
		\item There exists $\sigma>0$ such that $\expect[\|\nabla f(X, \phi^*)\|^2) \le \sigma^2$
		\item The function $\phi \mapsto f(X, \phi)$ is almost surely twice differentiable with Lipschitz continuous Hessian $H f$, i.e. $\forall \phi_1, \phi_2 \in \Phi$:
		\begin{equation*}
			 \expect(\|(H f)(X, \phi_1) - (H f)(X, \phi_2)\| ) \le C' \|\phi_1 - \phi_2\|
			\end{equation*}
		\item There exists $\tau>0$ such that $\expect[\|\nabla f(X, \phi^*)\|^4] \le \tau^4$ and there exists a positive definite operator $\Sigma$ such that $\expect[\nabla f(X, \phi^*) \otimes \nabla f(X, \phi^*)] \preccurlyeq \Sigma$
		\item The function $\Ialone$ is Lipschitz continuous
	\end{enumerate}
	\label{ass:sgd}
\end{assumption}

It should be noted that, though relatively standard, these assumptions are also quite strong, particularly the assumption of strong convexity of $f$, and may well not hold in practice.
In short, the stochastic gradient scheme used in optimizing the bounds may only converge toward a local optimum of the bound $\phi^\dag$, rather than the global optimum $\phi^*$.
When this happens the behavior and rates of convergence will generally be the same, but the error breakdown will become
\begin{subequations}
	\begin{align}
	&\left\| \hat{\mu}(d,\phi_K) - \eig(d) \right \|_{2} \nonumber\\
	&\quad\quad\quad\quad\quad\le \left\|  \hat{\mu}(d,\phi_K) - \mathcal{B}(d,\phi_K)\right\|_{2} \\
	&\quad\quad\quad\quad\quad\phantom{\le}+ \left\| \mathcal{B}(d,\phi_K)- \mathcal{B}(d,\phi^\dag) \right\|_{2} \\
	&\quad\quad\quad\quad\quad\phantom{\le}+ \left| \mathcal{B}(d,\phi^\dag) - \eig(d) \right|.
	\end{align}
\end{subequations}
where 
\begin{align*}
\left| \mathcal{B}(d,\phi^\dag) - \eig(d) \right| \ge \left| \mathcal{B}(d,\phi^*) - \eig(d) \right|.
\end{align*}

We now present our proof for the result, repeating the Theorem itself for convenience.
\themain*
\paragraph{Proof of Theorem~\ref{the:main}}
\begin{proof} We begin by establishing the uniform convergence of $\ialone_N(\phi)$ to $\Ialone(\phi)$, for which we simply use the $L^2$ weak law of large numbers. Specifically, we let $Y_n = f(X_n, \phi)$ and $\varepsilon_N(\phi) = \|\ialone_N(\phi) - \Ialone(\phi)\|_{2}$, then
\begin{align}
	\varepsilon_N^2(\phi) &= \expect\left(\left[ \frac{1}{N}\sum_{n=1}^N (Y_n - \expect Y_n)\right]^2\right) \\
	&= \expect\left(\frac{1}{N^2}\sum_{n=1}^N (Y_n - \expect Y_n)^2\right) \\
	&= \frac{1}{N^2} \cdot N \text{Var}(Y_n) \\
	&\le \frac{1}{N} \sup_{\phi\in\Phi}\|f(X_1,\phi)\|_{2}^2
\end{align}
which is bounded by assumption. Thus
\begin{equation}
	\sup_{\phi\in\Phi} \varepsilon_N(\phi) = \mathcal{O}(N^{-1/2})
\end{equation}
as required.

We turn now to the stochastic gradient descent convergence.
We begin by applying Theorem 3 of \cite{moulines2011} using points 1-5 of Assumption~\ref{ass:sgd} to give
\begin{equation}
\label{eq:l2conv}
	\|\phi_K - \phi^*\|_{2} = \mathcal{O}(K^{-1/2})
\end{equation}
and (see \cite{moulines2011} page 4)
\begin{equation}
	\expect \Ialone(\phi_K) - \Ialone(\phi^*) = \mathcal{O}(K^{-1/2}).
\end{equation}
To establish $L^2$ convergence of the function values, it remains to control the variance of $\Ialone(\phi_K)$. We now invoke point 6 of Assumption~\ref{ass:sgd} to see that, for some constant $B$ (namely the Lipschitz constant for $\Ialone$),
\begin{align}
	\text{Var}[\Ialone(\phi_K)] &= \expect\left[\left(\Ialone(\phi_K) - \expect\left[\Ialone( \phi_K)\right]\right)^2\right] \\
	&\le \expect\left[(\Ialone(\phi_K) - \Ialone(\expect \phi_K))^2\right] \\
	&\le B^2 \expect\left[(\phi_t - \expect\phi_t)^2\right] \\
	& \le B^2 \|\phi_K - \phi^*\|^2_{2}.
\end{align}
By \eqref{eq:l2conv} we conclude $\sqrt{\text{Var}[\Ialone(\phi_K)]} = \mathcal{O}(K^{-1/2})$. Thus $\Ialone(\phi_K)$ converges in $L^2$ at the required rate.

Finally, if $\epsilon_K = \|\ialone_K(\phi_K) - \Ialone(\phi^*)\|_{2}$ then
\begin{align*}
	\epsilon_K &\le \|\ialone_K(\phi_K) - \Ialone(\phi_K)\|_{2} + \|\Ialone_K(\phi_K) - \Ialone(\phi^*)\|_{2} \\
	& \le \|\ialone_K(\phi_K) - \Ialone(\phi_K)\|_{2}+ \sup_{\phi\in\Phi}\|\ialone_K(\phi)-\Ialone(\phi)\|_{2} \\
	&= \mathcal{O}(N^{-1/2}+K^{-1/2}) \\
	&= \mathcal{O}(T^{-1/2})
\end{align*}
as required.
\end{proof}

Finally, we discuss the necessary extensions for $\Iml$. The assumptions of the Theorem are subtly different in this case. Specifically, we require Assumption~\ref{ass:sgd} to hold for the integrand of $\mathcal{F}$ rather than the integrand of $\Iml$, where $\mathcal{F}(d, \phi) = -\mathbb{E}[\log \qm(y|d) + \log \ql(y|\theta,d)] + C$ is the loss function that we use to train $\phi$, and require $\Iml$ to be Lipschitz continuous in $\phi$.

The Monte Carlo error is no different in this setting. However, $\phi^*$ is optimal with respect to $\mathcal{F}(d,\phi)$ rather than $\Iml$ and the asymptotic bias term is $|\Iml(d,\phi^*) - \eig(d)| \le \mathcal{F}(d,\phi^*)$ by Lemma~\ref{lemma:ml}. For the optimization term, we have from equation~\eqref{eq:l2conv} that $\|\phi_K - \phi^*\|_2 = \mathcal{O}(K^{-1/2})$. Then by the Lipschitz assumption on $\Iml$, we have $\|\Iml(d,\phi_k) - \Iml(d,\phi^*)\|_2 = \mathcal{O}(K^{-1/2})$. The rest of the proof now goes through as above.



\section{Related work}
\label{sec:app:relatedwork}
  
In this section, we provide a more detailed discussion of existing techniques for EIG estimation to complement Sec.~\ref{sec:experimentseig} in the main text.

One established approach is to use a \textbf{Laplace approximation} to the posterior to make fast approximations of EIG \citep{lewi2009sequential,long2013}\vspace{-10pt}
\begin{equation}
\label{eq:laplace}
\ilap(d) \triangleq \frac{1}{N}\sum_{n=1}^{N} \left[\entropy{p(\theta)} - \entropy{q(\theta|y_n,d)}\right]
\end{equation}
where $q(\theta|y_n,d)$ is a Laplace approximation to $p(\theta|y_n, d)$ 
that is computed once for each $y_n \sim p(y|d)$. 

\citet{kleinegesse2018efficient} recently suggested an implicit likelihood approach that directly approximates the ratio $r(d,\theta,y) = p(y|\theta,d)/p(y|d)$ using samples from $p(y|\theta,d)$ and $p(y|d)$ and the \textbf{Likelihood-Free Inference by Ratio Estimation (LFIRE)} method suggested by~\citep{dutta2016likelihood}, which is itself based around logistic regression.
This yields the estimator
\begin{equation}
\label{eq:lfire}
\ilfire(d) \triangleq \frac{1}{N} \sum_{n=1}^N \log \hat{r}(d,\theta_n,y_n)
\end{equation}
where $\log \hat{r}(d,\theta_n,y_n)$ is estimated separately for each pairs of samples $y_n, \theta_n$.

In principal one could also  exploit the equivalence between EIG and MI and use other existing MI estimation methods, a number of which were recently summarized by~\cite{poole2018variational}.
Of particular note,~\citet{mine} use a bound on \acrshort{MI} in the context of generative adversarial neural network training that is based on the \textbf{Donsker-Varadhan (DV)} representation of the KL divergence \citep{donsker1975asymptotic}. Specifically, they introduce a parametrized approximation $T(y,\theta|d,\phi)$ to
$\log\frac{p(y,\theta|d)}{p(\theta)p(y|d)}$ and then optimize the lower bound
\begin{align}
\label{eq:dv}
\Idv(d) \triangleq\,\, &\expect_{p(y,\theta|d)}[T(y,\theta|d,\phi)] -\log\left(\expect_{p(\theta)p(y|d)}[e^{T(y,\theta|d,\phi)}]\right).
\end{align}
The estimator $\idv$ is then produced in an analogous manner to $\ipost$.

The EIG has been applied by a number of authors in specific contexts. 
For instance, the EIG has been used to formulate acquisition functions in Bayesian optimization \cite{hernandez2014}.
More recently, \citet{ma2018eddi} used an \acrshort{EIG}-type objective to select features rather than designs for a partial VAE model. The EIG estimation exploits the model structure of the partial VAE. Additionally, and in contrast to this paper, approximations learned using the ELBO are used rather than approximations that are trained using variational objectives that are directly tied to EIG estimation. For further discussion on the implications of using the ELBO (i.e.~the reverse KL divergence) in EIG estimation settings, see Appendix~\ref{sec:reverseforward}.

As mentioned previously, mutual information bounds are of interest in traditional signal processing \cite{ba} and of increasing interest in the deep learning community \cite{poole2018variational}---although to the best of our knowledge they have not been applied to \acrshort{OED} before. Interestingly, it is lower bounds that are of primary importance in the deep learning setting because of the interplay between MI estimation and the subsequent gradient-based optimization over parameters. This is in contrast to this work, in which we maximize EIG over designs using Bayesian optimization---allowing the use of estimators such as $\iml$ that are not, in expectation, bounds.


\section{Experiment details}
\label{sec:appendix:exp}

\paragraph{Computing} All experiments were run on a machine with 32818560 kB memory, 8 Intel(R) Core(TM) i7-6700 CPU @ 3.40GHz processors, running Fedora 28, Python 3.6.8, Pytorch 1.1.0. To reproduce the results presented in the paper, see \texttt{https://github.com/ae-foster/pyro/tree/vboed-reproduce}. The methods in this paper form part of Pyro's OED support, the documentation for which is provided at \texttt{http://docs.pyro.ai/en/stable/contrib.oed.html}.

\subsection{\acrshort{EIG} estimation accuracy}
\label{sec:app:eigaccuracy}

\paragraph{A/B test} We consider a classical A/B test, commonly used in marketing and design applications.
Here the experiment design is the choice of group sizes: $n$ participants are split between groups A and B of size $n_A$ and $n - n_A$, respectively. For each participant we measure a continuous response $y$. We consider a linear data analysis model
\begin{align}
\begin{split}
	\theta \sim N(0, \Sigma_\theta) \qquad
	y | \theta, d \sim N(X_d\theta, I)
\end{split}
\end{align}
where $X_d$ is the $n \times 2$ design matrix with $(1~0)$ 
for the first $n_A$ rows and $(0~1)$
for the remainder.

In this example we set the number of participants to be $n=10$ with 11 designs ($n_A= 0, ..., 10$) and the prior covariance matrix to be
\begin{align}
	\Sigma_{\theta} &= \begin{pmatrix}
		10^2 & 0 \\
		0 & 1.82^2
	\end{pmatrix}
\end{align}
We chose families of variational distributions that include the true posterior (or true marginal). For the amortised posterior, we set $\phi = (A, \Sigma_\text{p})$ with $\phi$ trained separately for each $d$ and let
\begin{align}
	\qp(\theta | y, d, \phi) &\sim N(Ay, \Sigma_\text{p}) 
\end{align}
where $A$ is a $10 \times 2$ matrix and $\Sigma_\text{p}$ is positive definite. For the marginal, we simply take $\phi = (\mu_\text{m}, \Sigma_\text{m})$ and
\begin{equation}
	\qm(y|d,\phi) \sim N(\mu_\text{m}, \Sigma_\text{m}).
\end{equation}

For NMC and Laplace, no variational families need to be specified.

For LFIRE, we used a parametrization $\phi = (b, \delta, \Lambda)$ and used the ratio estimate
\begin{equation}
	\log \hat{r}(y|\theta, d,\phi) = b - (y - \delta)^T\Lambda(y - \delta)
\end{equation}
where $\Lambda$ is positive definite. This form was chosen to mimic the approximation made by the posterior method, and so reduce the effect of architecture on performance.

For DV, we used a similar critic, namely we set $\phi = (A, \Lambda)$ and 
\begin{equation}
	T(y,\theta|d,\phi) = -(\theta - Ay)^T\Lambda(\theta-Ay)
\end{equation}
where $\Lambda$ is positive definite.

The ground truth $\eig(d)$ was computed analytically. In Table~\ref{tab:abserrors}, each estimator was allowed 10 seconds computation.

\paragraph{Preference} 
\label{sec:locfind}

We  consider searching for an agent's utility indifference point, using responses that are both \textit{censored} and \textit{corrupted} with non-uniform noise.
Let $d \in \mathbb{R}$ and
\begin{align}
\begin{split}
	\theta &\sim N(\mu_\theta, \sigma_\theta^2) \\
	\eta|\theta,d &\sim N(d - \theta, \sigma_\eta^2(1+|d|)^2) \\
	y &= f(\eta)
\end{split}
\end{align}
where
\begin{align}
	f : \R &\to [\epsilon, 1 - \epsilon] \\
	x &\mapsto \begin{cases}
		\epsilon & \text{ if } x \le \text{logit}(\epsilon) \\
		1 - \epsilon & \text{ if } x \ge \text{logit}(1-\epsilon) \\
		\frac{1}{1 - e^{-x}} & \text{ otherwise}
	\end{cases}
	\label{eq:censoredsigmoid}
\end{align}
and $\text{logit}(p) = \log p - \log (1-p)$.

For this example we set $\mu_\theta=-20,\sigma_\theta=20$ and $\sigma_\eta=1$. We took designs on a linearly spaced grid in $[-80,80]$.
For the variational family for the posterior, we took $\phi = (w, \sigma, \mu_{0},\sigma_{0},\mu_{1},\sigma_{1})$ and then
\begin{align}
	\qp(\theta | y, d, \phi) &\sim N(\mu_\text{p}, \sigma_\text{p}^2) \qquad 
	\text{where }\qquad \hat{\eta} = d - \text{logit}(y) \\
	\begin{split}
	\mu_\text{p} &= w\hat{\eta} + (1-w)\mu_\theta  + \mu_0 \1_{\{y=\epsilon\}}  + \mu_1 \1_{\{y=1-\epsilon\}} \end{split}\\
	\begin{split}
	\sigma_\text{p}^2 &= \sigma^2  + \sigma_0^2 \1_{\{y=\epsilon\}} + \sigma_1^2 \1_{\{y=1-\epsilon\}} \end{split}
\end{align}
For the marginal, we simply took $\phi = (\mu_\text{m}, \sigma_\text{m})$ and
\begin{equation}
	\qm(y|d,\phi) \sim f\#N(\mu_\text{m}, \sigma_\text{m}^2).
\end{equation}
where $\#$ denotes the push-forward measure. We note that this variational family contains the true marginal.

For LFIRE, we used the parametrization $\phi = (b, b_0, b_1, \delta, \lambda)$ with ratio estimate
\begin{align}
    \hat{\eta} &= d - \text{logit}(y) \\
	\begin{split}
    \log \hat{r}(y|\theta,d,\phi) &= b - \lambda(\hat{\eta}-\delta)^2 + b_0 \1_{\{y=\epsilon\}} + b_1 \1_{\{y=1-\epsilon\}} \end{split}
\end{align}

For DV, the critic had parametrization $\phi = (b_0, b_1, \delta_i, \delta_0, \delta_1, \lambda_i, \lambda_0, \lambda_1)$ and we set
\begin{align}
\hat{\eta} &= d - \text{logit}(y) \\
\lambda &= \lambda_i + \lambda_0 \1_{\{y=\epsilon\}} + \lambda_1 \1_{\{y=1-\epsilon\}} \\
\delta &= \delta_i + \delta_0 \1_{\{y=\epsilon\}} + \delta_1 \1_{\{y=1-\epsilon\}} \\
	\begin{split}
    T(y,\theta|d,\phi) &= - \lambda(\hat{\eta}-\delta)^2 + b_0 \1_{\{y=\epsilon\}} +\, b_1 \1_{\{y=1-\epsilon\}} \end{split}
\end{align}
Both these forms were chosen to minimize the differences between the functional forms used for different methods.

The ground truth $\eig(d)$ was computed by running the marginal method, which is statistically consistent for this example because the true marginal is contained in the variational family, to convergence. The posterior and Laplace 
methods are both asymptotically biased (see Figure~\ref{fig:locfind}) and in this case both make the same (Gaussian) distributional assumption. The posterior method, however, produces better \acrshort{EIG} estimates. For the benchmarking results in Table~\ref{tab:abserrors}, 10 seconds computation was allowed.

\begin{figure}[t]
  \centering
  {\includegraphics[width=0.75\textwidth]{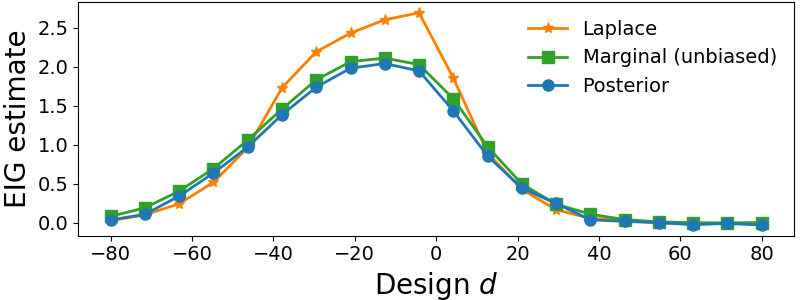}}
  \caption{EIG curves for the Preference example, with estimators run until variance is negligible and iterates of $\phi$ are stable to highlight the asymptotic bias.}
  \label{fig:locfind}
\end{figure}

\paragraph{Mixed Effects Regression}

We consider BOED for a mixed effects regression model with a non-linear linking function that
will also serve as the basis for the adaptive experiment we run in Sec.~\ref{sec:mturk}.
This class of models is commonly used for analyzing data in a variety of scientific disciplines,
where including nuisance variables can be a critical component of the model.
In our adaptive experiment, the nuisance variables---i.e.~the random effects---are used
to account for the variability of individual human participants. 
Because of the presence of nuisance variables these implicit likelihood models represent
a significant challenge for BOED.

We begin by describing the experiment set-up. Participants were presented with a question of the form seen in Figure~\ref{fig:screenshot} with the possible images shown in Figure~\ref{fig:sprites}. There were two image feature dimensions with 3 levels each. A single image $i$ could therefore be represented as a $1 \times 6$ matrix $X_i$ with two entries 1 and the rest 0. With the left image $i_1$ and right image $i_2$, the question was represented as $X_d = X_{i_1} - X_{i_2}$ encoding the assumed left-right symmetry. We then considered a model for the $i$th participant
\begin{align}
	\theta &\sim N(0, \Sigma_\theta) \\
	\sigma_\psi^{-2} &\sim \Gamma(\alpha_\psi, \beta_\psi) \\
	\psi_i|\sigma_\psi &\sim N(0, \sigma_\psi^2 I_6) \\
	\sigma_k^{-2} &\sim \Gamma(\alpha_k, \beta_k) \\
	\log k_i|\sigma_k &\sim  N(0, \sigma_k^2) \\
	\eta|\theta,\psi_i,k_i,d &\sim N(k_i(X_d\theta + X_d\psi_i), \sigma_\eta^2) \\
	y &= f(\eta)
\end{align}
where $f$ is the censored sigmoid defined in \eqref{eq:censoredsigmoid} and $i\in \{1, ..., 8\}$ as there were 8 different participants.

The actual prior values of the parameters used were
\begin{align}
    \Sigma_\theta &= 100 I_6 \qquad \qquad  \sigma_\eta = 10 \\
    \alpha_\psi &= \beta_\psi = 
	\alpha_k = \beta_k = 2 
\end{align}

\begin{figure*}[p]\centering
	\includegraphics[scale=.3]{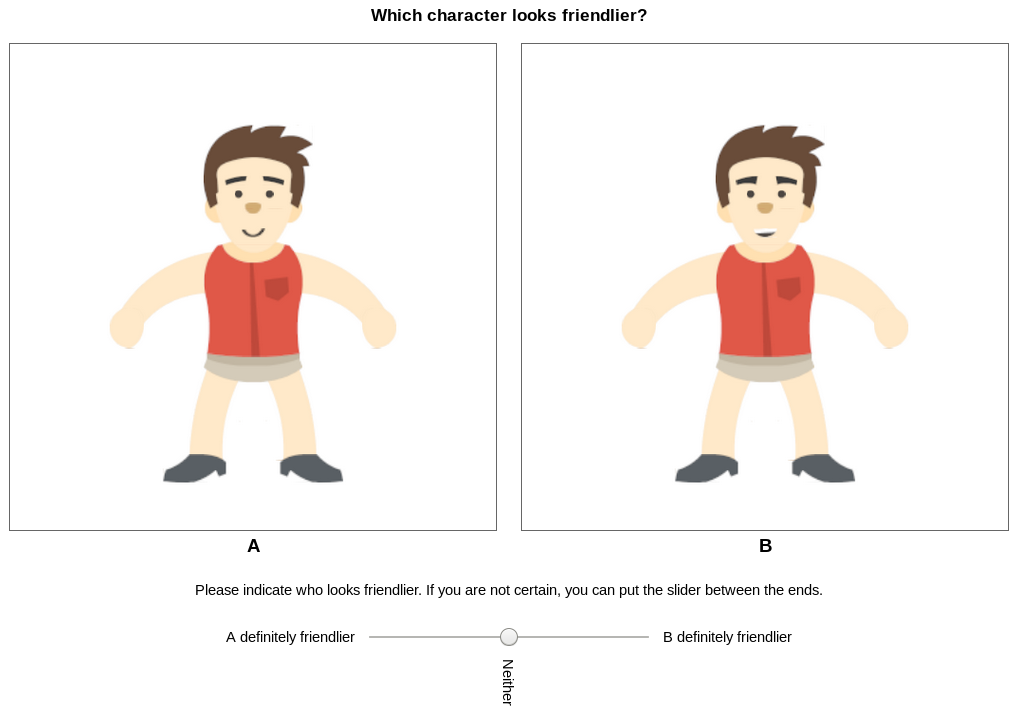}
        \caption{A screenshot of the question answering interface used by human participants in the adaptive experiment in Sec.~\ref{sec:mturk}.}
	\label{fig:screenshot}
\end{figure*}

\begin{figure*}[p]
\centering
\includegraphics[scale=0.15]{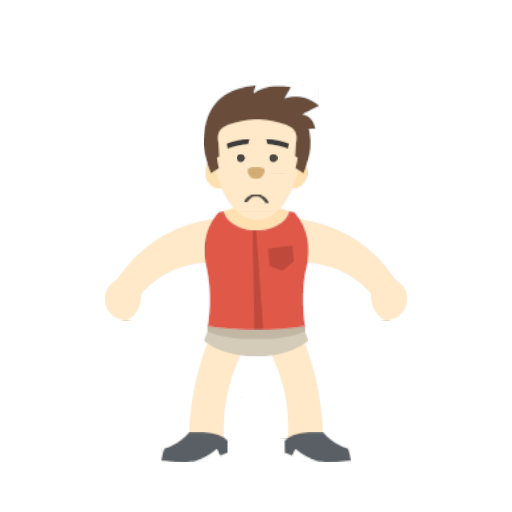}
\includegraphics[scale=0.15]{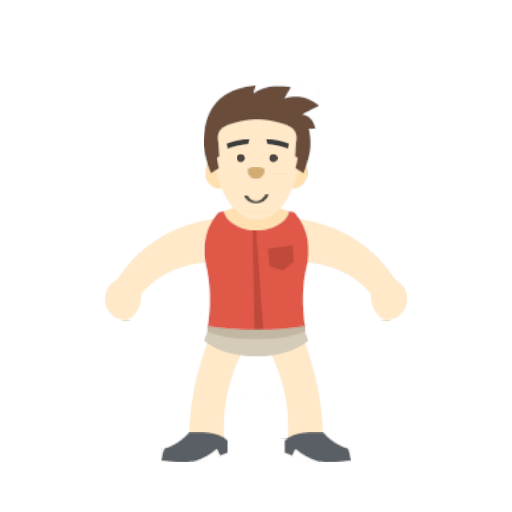}
\includegraphics[scale=0.15]{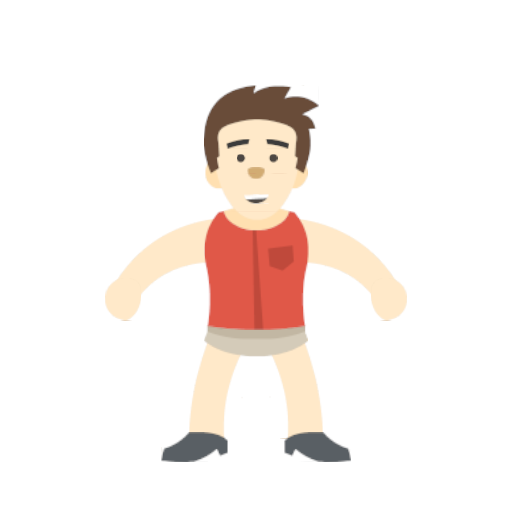}
\includegraphics[scale=0.15]{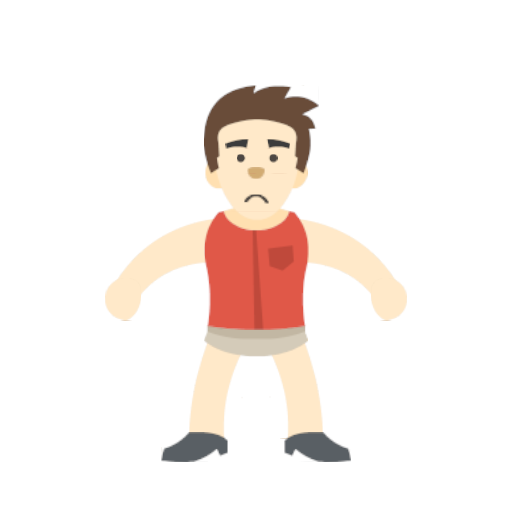}
\includegraphics[scale=0.15]{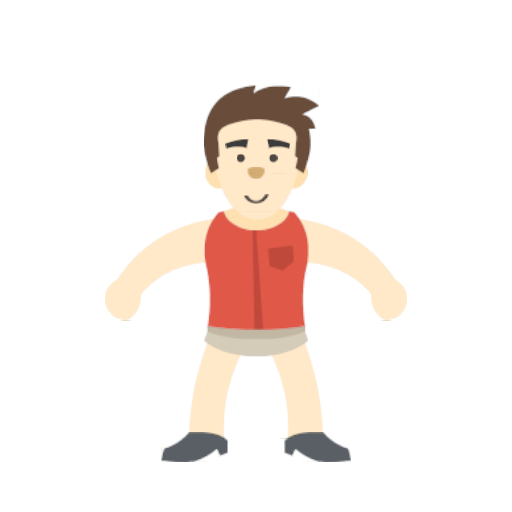}
\includegraphics[scale=0.15]{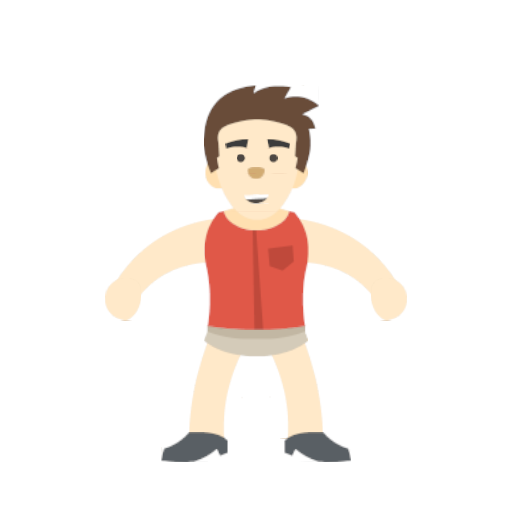}
\includegraphics[scale=0.15]{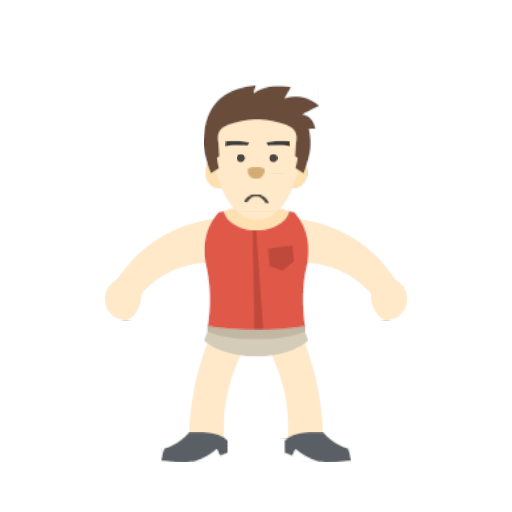}
\includegraphics[scale=0.15]{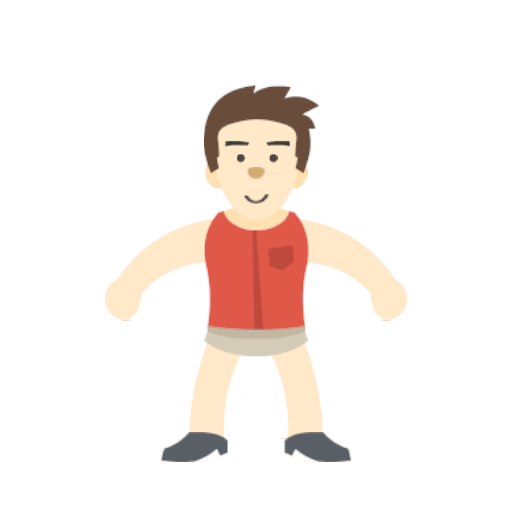}
\includegraphics[scale=0.15]{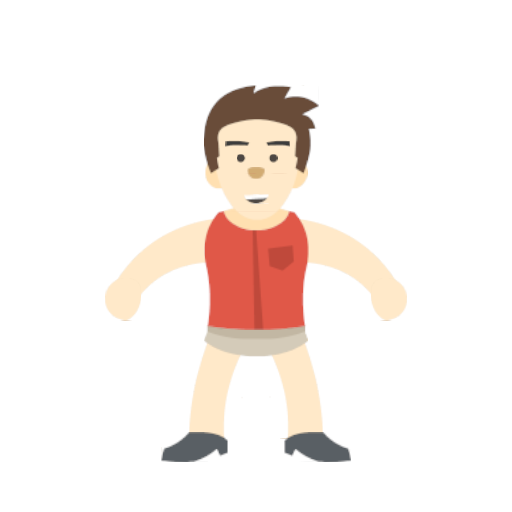}
        \caption{The nine characters we used in the adaptive experiment in Sec.~\ref{sec:mturk}. They vary along two feature dimensions: the mouth (smile, frown, showing teeth) and eyebrows.}
	\label{fig:sprites}
\end{figure*}

We begin by discussing the variational families used to estimate the \acrshort{EIG}.

For the posterior estimator of \acrshort{EIG}, we took $\phi = (A, \Sigma_\text{p})$ and 
\begin{align}
	\hat{\eta} &= \text{logit}(y) \\
	\qp(\theta|y,d,\phi) &\sim N(A\hat{\eta}, \Sigma_\text{p})
\end{align}

For the marginal + likelihood estimator, we set $\phi = (\mu_\text{m}, \sigma_\text{m}, \mu_\ell, \sigma_\ell, \xi)$ and took
\begin{align}
	\qm(y|d,\phi) &\sim f\#N(\mu_\text{m}, \sigma_\text{m}^2) \\
	\ql(y|\theta,d,\phi) &\sim f\#N(e^\xi X_d\theta + \mu_\ell, \sigma_\ell^2)
\end{align}

For LFIRE, we used $\phi = (b, \delta, \lambda)$ and then took
\begin{align}
	\hat{\eta} &= \text{logit}(y) \\
	\log \hat{r}(y|\theta,d,\phi) &= b - \lambda(\hat{\eta}-\delta)^2
\end{align}

For DV, we used $\phi = (\lambda,\xi)$ and
\begin{align}
	\hat{\eta} &= \text{logit}(y) \\
	T(y,\theta|d,\phi) &=  - \lambda(\hat{\eta} - e^\xi X_d\theta)^2
\end{align}

For benchmarking, we computed the ground truth using a variant of \acrshort{NMC}. Specifically, we note that
\begin{align}
	p(y|d) &= \expect_{p(\theta,\psi,k)}[p(y|\theta,\psi,k,d)] \\
	p(y|\theta,d) &= \expect_{p(\psi,k)}[p(y|\theta,\psi,k,d)]
\end{align}
and for this model, we can sample directly from $p(\psi,k)$. These identities allow us to estimate the marginal and likelihood by \acrlong{MC}, and then combine in a \acrshort{NMC} estimator for $\eig(d)$. Whilst inefficient, this estimator is statistically consistent.

We allowed 60 seconds computation per estimator to compute the results of Table~\ref{tab:abserrors}. Encouragingly, we find that our variational estimators outperform the LFIRE and DV 
baselines on this model and exhibit low errors even though
they both make suboptimal distributional assumptions about the posterior/marginal.

\paragraph{Extrapolation} 
We consider designing experiments to reduce posterior uncertainty in the model prediction at another point in design space---a point that we cannot experiment on directly. For this example, we take $\psi \sim N(\mu_\psi, \Sigma_\psi)$ and
\begin{align}
\begin{split}
\nonumber
	\theta | \psi  &\sim \text{Bernoulli}(\text{logit}^{-1}(X_\theta \psi)) \\
	y | \psi, d  &\sim \text{Bernoulli}(\text{logit}^{-1}(X_d \psi))
\end{split}
\end{align}
where $X_\theta = \begin{pmatrix}
1 & -\tfrac{1}{2}
\end{pmatrix}$ and $X_d = \begin{pmatrix}
-1 & d
\end{pmatrix} \text{ for } d \in \mathbb{R}$. Interestingly, this model admits efficient sampling of $y,\theta\sim p(y,\theta|d)$ but \textit{not} $y \sim p(y|\theta,d)$. Therefore, whilst the posterior, marginal + likelihood and DV methods are all applicable, LFIRE is not.

For the posterior method we set $\phi = (l_0, l_1)$ and
\begin{align}
	l_\text{p}(y) &= l_1y + l_0(1-y) \\
	\qp(\theta|y,d,\phi) &\sim \text{Bernoulli}(\text{logit}^{-1}(l_\text{p}(y))).
\end{align}
We computed the prior entropy, which is not analytically tractable here, using a \acrshort{MC} estimator, noting that $\theta$ has a finite sample space.

For the marginal + likelihood method, we let $\phi = (l, l_0, l_1)$ and then
\begin{align}
	\qm(y|d,\phi) &\sim \text{Bernoulli}(\text{logit}^{-1}(l)) \\
	l_\ell(\theta) &= l_1\theta + l_0(1-\theta) \\
	\ql(y|\theta,d,\phi) &\sim \text{Bernoulli}(\text{logit}^{-1}(l_\ell(\theta))).
\end{align}

Finally, for DV, we let $\phi = (w_y, w_\theta, w_{y\theta})$ and took
\begin{equation}
	T(\theta,y|d,\phi) = w_yy + w_\theta \theta + w_{y\theta}y\theta.
\end{equation}

The ground truth \acrshort{EIG} was computed using \acrshort{MC}, noting that the sample spaces for $y,\theta$ are finite in this example. 10 seconds computation per methods was allowed for the results in Table~\ref{tab:abserrors}.

\subsection{End-to-end sequential experiments}

\paragraph{Mechanical Turk experiment}
\label{sec:mturkmodel}

\begin{figure}
	\begin{center}
		\includegraphics[width=.4\textwidth]{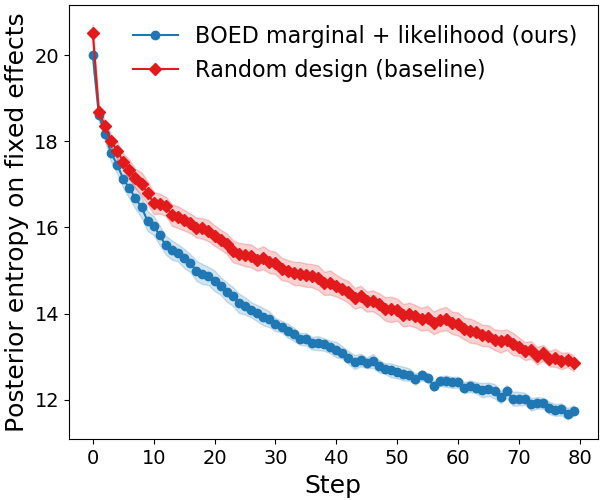}
		\caption{Evolution of the posterior entropy of the fixed effects in the Mechanical Turk experiment in Sec.~\ref{sec:mturk} with simulated data. We depict the mean and $\pm1$ std.~err.~from 10 experimental trials.}
		\label{fig:turkexperimentsimul}
	\end{center}
\end{figure}

We begin by describing the experiment itself. Participants were presented with a question of the form seen in Figure~\ref{fig:screenshot} with the possible images shown in Figure~\ref{fig:sprites}. There were two image feature dimensions with 3 levels each. A single image $i$ could therefore be represented as a $1 \times 6$ matrix $X_i$ with two entries 1 and the rest 0. With the left image $i_1$ and right image $i_2$, the question was represented as $X_d = X_{i_1} - X_{i_2}$ encoding the assumed left-right symmetry.

The model and \acrshort{EIG} estimation were the same as the mixed effects model in Sec.~\ref{sec:app:eigaccuracy}. When optimizing the \acrshort{EIG} to select designs $d_t$, we estimated \acrshort{EIG} across all candidate designs. We allowed a 30s turnaround to learn the posterior from the previous data, estimate the EIG, select the next design, and present it to the user. We estimated the \acrshort{EIG} in parallel for all 36 designs to select the best design at each step. For each independent run of the experiment there were 8 participants, each answering 10 questions. This allowed the interplay between fixed effects and random effects to be apparent. 

Because we used this model to run an adaptive experiment, we required a variational family to learn the full posterior (over random effects and hyperparameters as well as $\theta$).

For the full variational inference of the posterior used when we receive actual data, we used a partial mean-field approximation. Specifically, we set $q(\theta,\sigma_\psi,(\psi_i)_{i=1}^8, \sigma_k,(k_i)_{i=1}^8)$ to be
\begin{align}
	\theta &\sim N(\mu_\theta, \Sigma_\theta) \\
	\sigma_\psi^{-2} &\sim \Gamma(\alpha_\psi, \beta_\psi) \\
	\psi_i|\theta &\sim N(A(\theta-\mu_\theta) + \mu_{\psi_i}, \Sigma_{\psi_i}) \\
	\sigma_k^{-2} &\sim \Gamma(\alpha_k, \beta_k) \\
	\log k_i &\sim  N(\mu_{k_i}, \sigma_{k_i}^2)
\end{align}
and we learned the variational parameters $\mu_\theta, \Sigma_\theta, \alpha_\psi, \beta_\psi, A, \mu_{\psi_i}, \Sigma_{\psi_i}, \alpha_k, \beta_k, \mu_{k_i}, \sigma_{k_i}$ by conventional (not amortized) variational inference. Note that, under this approximate posterior, $\theta$ is multivariate Gaussian so we can compute its entropy analytically.

Finally we ran an additional experiment identical to the first, but using simulated data rather than human responses. We took
\begin{equation}
	\theta = \begin{pmatrix}
		-30 & 30 & 0 & -12 & -6 & 18
	\end{pmatrix}.
\end{equation}
We simulated the random effects $\psi, k$ from the prior and used the prior value $\sigma_\eta=10$. The entropy results are presented in Figure~\ref{fig:turkexperimentsimul}. As expected, \acrshort{OED} decreases posterior uncertainty more quickly.

\subsection{Constant Elasticity of Substitution (CES) experiment}
We begin by describing the experiment set-up. The economic agent is presented with a sequence of designs $d$. Each designs comprises two baskets $\mathbf{x}$ and $\mathbf{x}'$ of goods. The agent then indicates which basket they prefer on a one-dimensional slider---they may indicate a strong preference, weak preference, or indifference.

To model the agent's responses, we use the CES utility model \cite{arrow1961capital} which defines a utility
\begin{equation}
	U(\mathbf{x}) = \left( \sum_i x_i^\rho \alpha_i\right)^{1/\rho}
\end{equation}
for a basket of goods $\mathbf{x}$. In this experiment, we took baskets $\mathbf{x} \in [0,100]^3$ representing non-negative quantities of three commodities.

Extending the preference example in the previous section, we assume the agent, when asked to compare baskets $\mathbf{x}$ and $\mathbf{x}'$ and indicate their preference on a slider, base their response on $U(\mathbf{x}) - U(\mathbf{x}')$. Specifically, we use the following likelihood model
\begin{align}
	\rho &\sim \text{Beta}(a_\rho, b_\rho) \\
	\bm{\alpha} &\sim \text{Dirichlet}(\bm{c}_{\bm{\alpha}}) \\
	\log u & \sim N(\mu_u, \sigma_u^2) \\
	\eta | \rho,\bm{\alpha},d &\sim N(u\cdot (U(\mathbf{x}) - U(\mathbf{x}')), \sigma_\eta^2 u^2 (1 + \|\mathbf{x} - \mathbf{x}'\|)^2) \\
	y &= f(\eta)
\end{align}

This represents a challenging experiment design problem for a number of reasons. First, for large values of $U(\mathbf{x}) - U(\mathbf{x}')$ the agent's response will be predictable gaining little information. For very different baskets ($\|\mathbf{x} - \mathbf{x}'\|$ large) the responses will be noisy indicating our intuition that it is more difficult to compare very different baskets. However, very similar baskets will have similar utilities and the agent will be predictably indifferent. Optimal designs therefore lie in a sweet spot where: i) baskets are similar to avoid high noise regions, but dissimilar enough to be informative; and ii) the difference in utility is close to 0 under the current posterior. \acrshort{OED} is able to trade off these considerations in a principled manner.

For this specific example we took
\begin{align}
	a_\rho &= b_\rho = 1  \qquad \qquad
	\bm{c}_{\bm{\alpha}} = (1, 1, 1) \\
	\mu_u &= 1  \qquad \qquad \qquad \;
	\sigma_u = 3 \\
	\sigma_\eta &= 0.005
\end{align}

To estimate the EIG, we used a marginal guide based on the one used in the preference example. Specifically, we set $\phi = (\mu_\text{m}, \sigma_\text{m}, p_0, p_1)$ and
\begin{align}
	r(y|d,\phi) &\sim f\#N(\mu_\text{m}, \sigma_\text{m}^2), \\
	\qp(y|d,\phi) &= \begin{cases}
		\epsilon & \text{with probability } p_0 \\
		1 - \epsilon & \text{with probability } p_1 \\
		r(y|d,\phi) & \text{with probability } 1 - p_0 - p_1
	\end{cases}
\end{align}
where $\#$ denotes the push-forward measure. This is simply a mixture of a discrete distribution on end-points with a sigmoid transformed Gaussian. 

To select designs, we used Bayesian optimization with a Matern52 kernel with lengthscale 20 and variance set empirically. Both $\imarg$ and $\inmc$ were allowed the same time budget to select designs and used an identical Bayesian optimization procedure. Random designs were chosen uniformly on $[0, 100]^6$.

To learn the posterior at subsequent steps we used a mean-field variational approximation with the same families as the prior. That is, we updated the parameters $a_\rho,b_\rho,\bm{c}_{\bm{\alpha}}, \mu_u,\sigma_u$ and left the structure otherwise intact. The RMSEs of Figure~\ref{fig:ces} were expectations over the posterior: $\left(\mathbb{E}_{p(\theta|d_{1:t},y_{1:t})}[\|\theta - \theta^*\|^2]\right)^{1/2}$.



\section{Additional experiments}

\subsection{Death process}
\label{sec:deathprocess}

\begin{figure*}[h]
	\centering
	    	\begin{subfigure}[b]{0.23\textwidth}
	    		\centering
	    		\includegraphics[width=\textwidth,trim={3mm 2mm 1.5mm 2mm}]{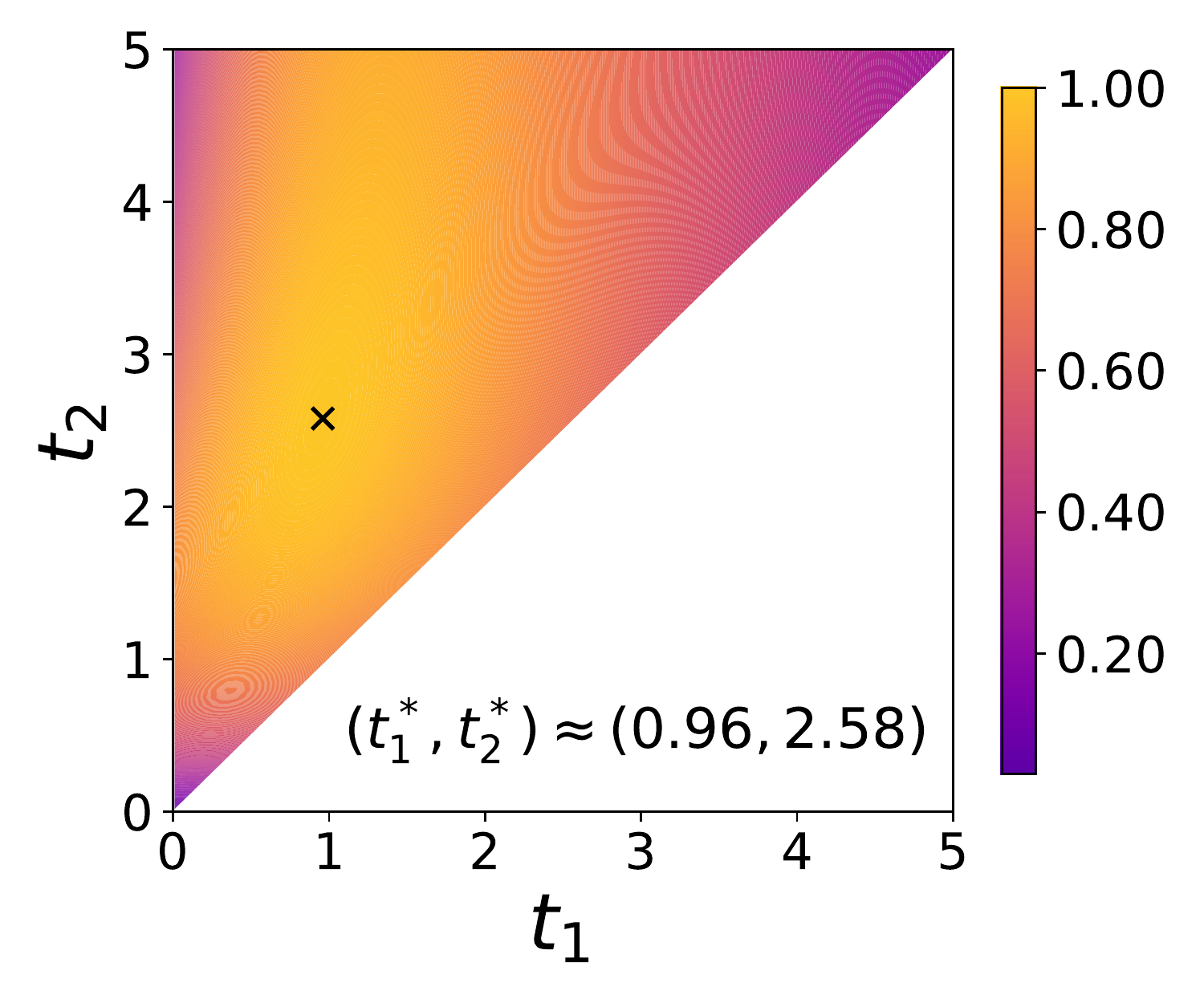}
	    		\caption{Exact EIG}
	    	\end{subfigure}
	    	~
	    	\begin{subfigure}[b]{0.23\textwidth}
	    		\centering
	    		\includegraphics[width=\textwidth,trim={3mm 2mm 1.5mm 2mm}]{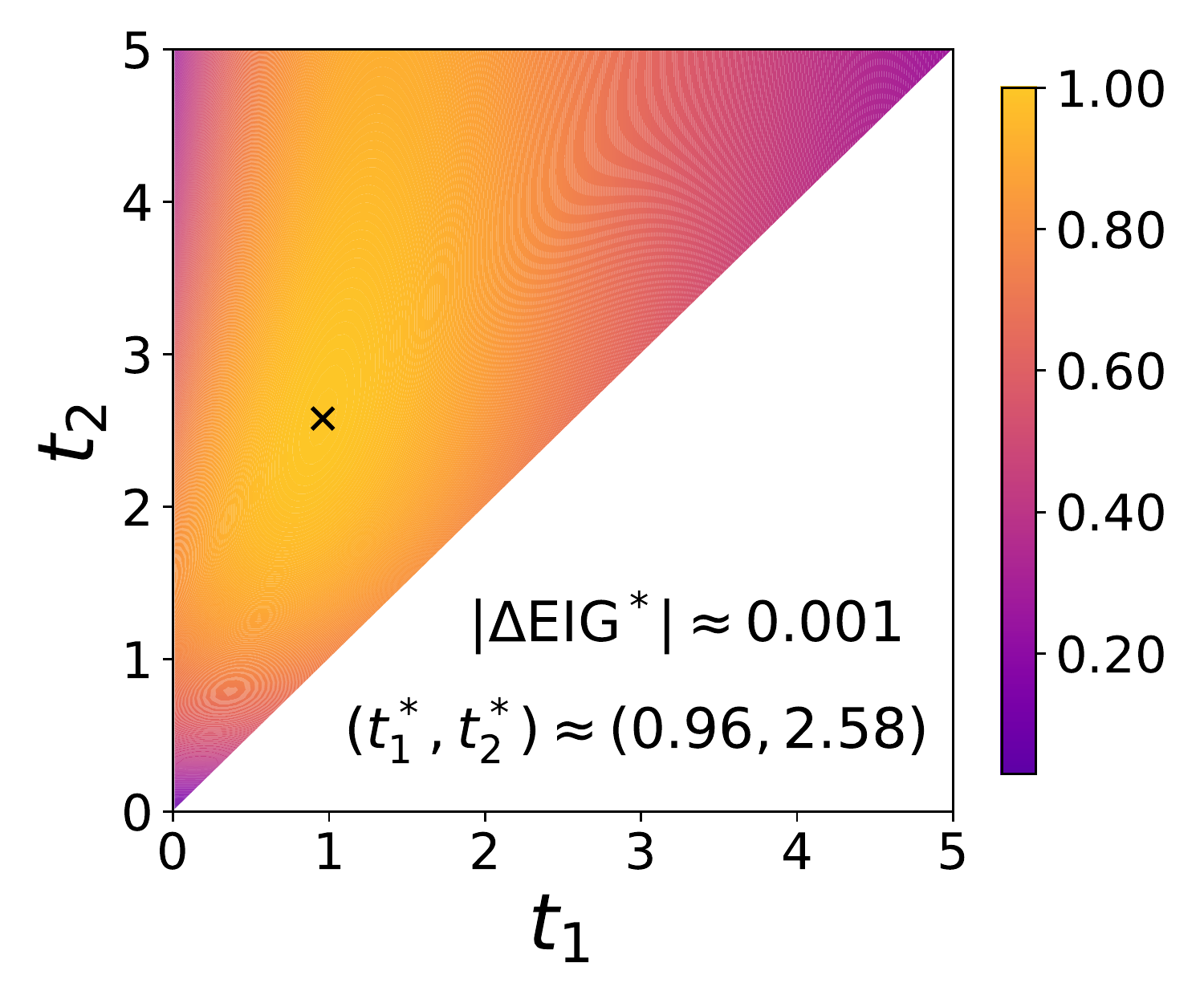}
	    		\caption{Posterior LogNormal}
	    	\end{subfigure}
	    	~
	    	\begin{subfigure}[b]{0.23\textwidth}
	    		\centering
	    		\includegraphics[width=\textwidth,trim={3mm 2mm 1.5mm 2mm}]{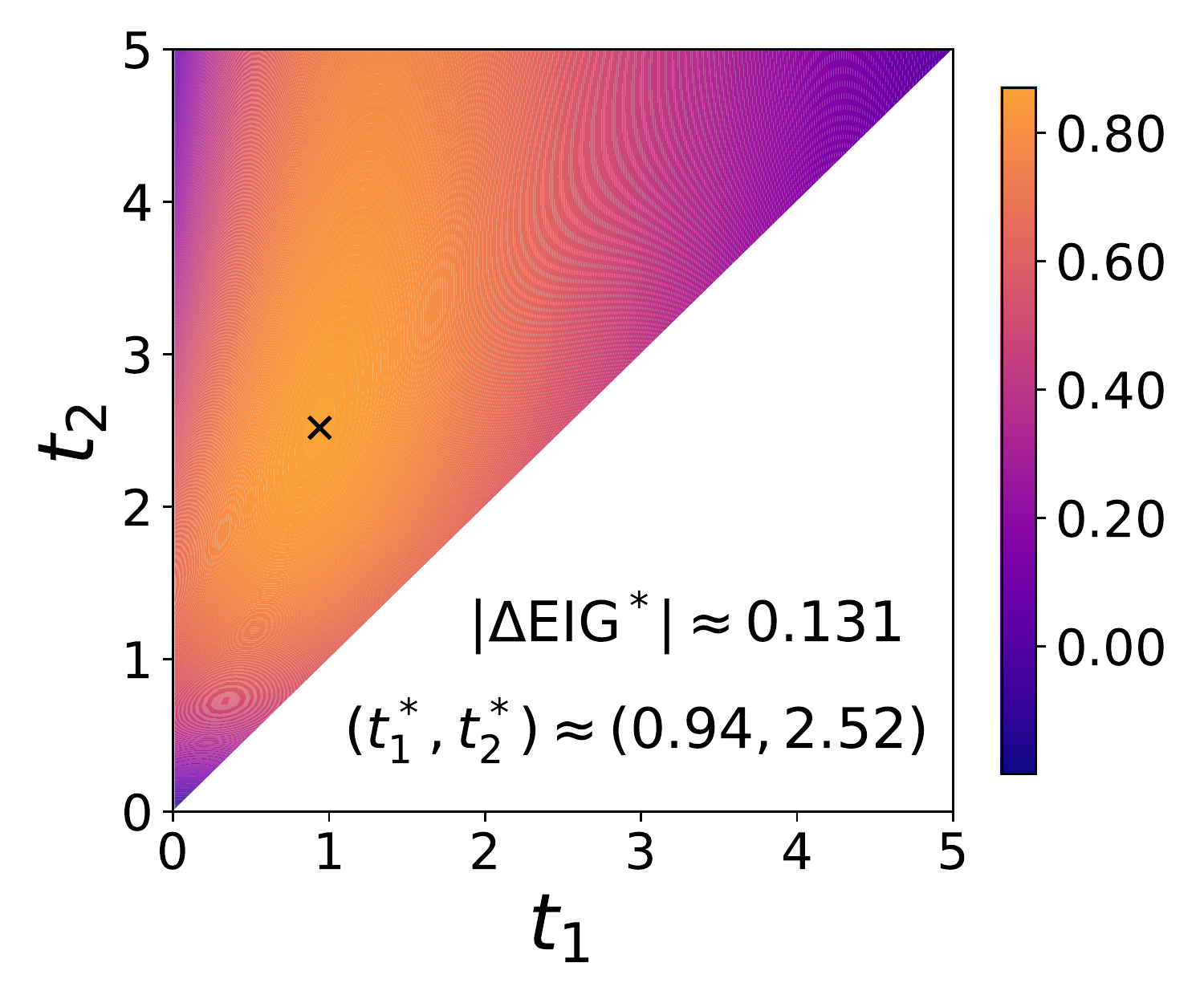}
	    		\caption{Truncated Normal}
	    	\end{subfigure}
	    	~
	    	\begin{subfigure}[b]{0.23\textwidth}
	    		\centering
	    		\includegraphics[width=\textwidth,trim={3mm 2mm 1.5mm 2mm}]{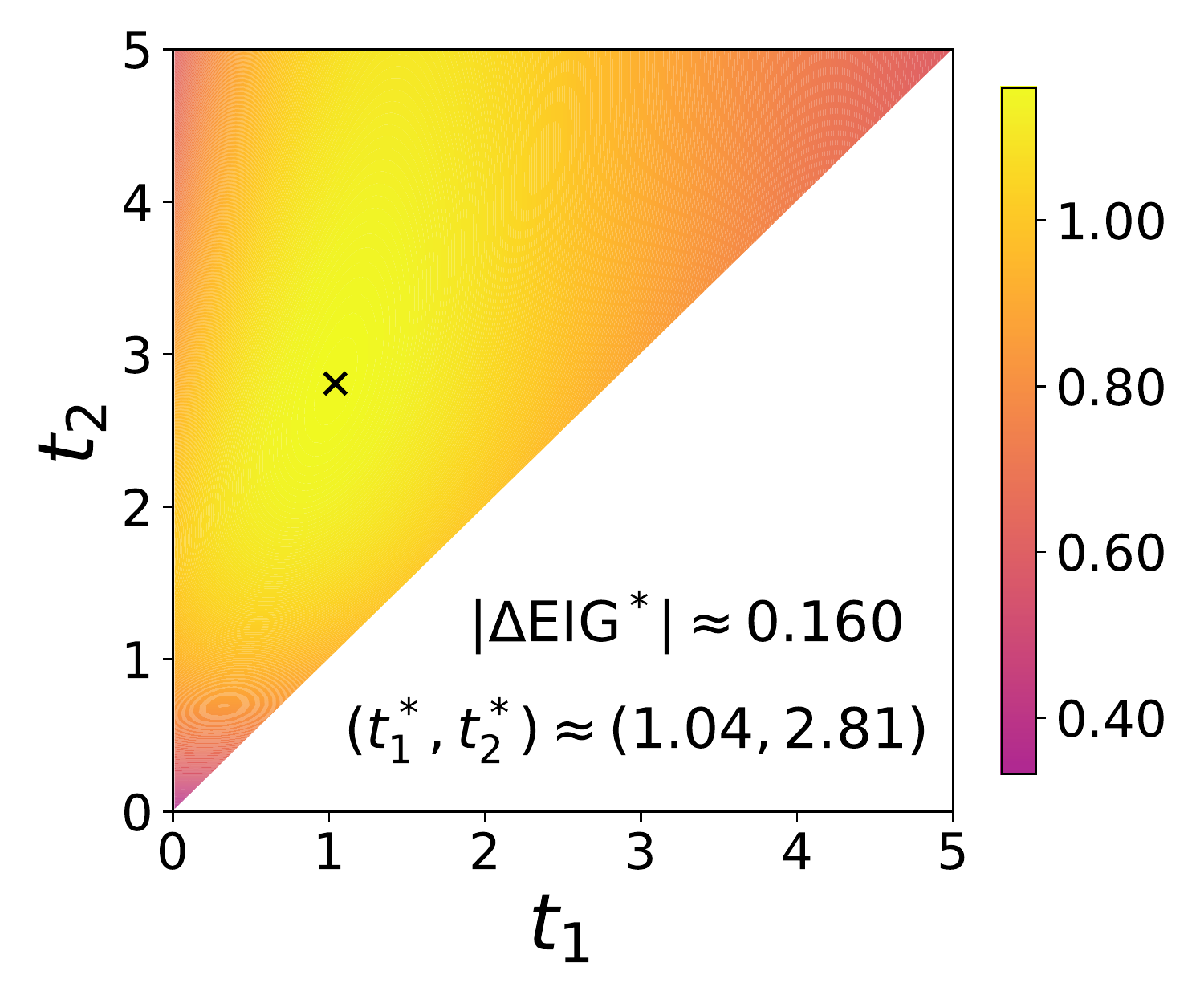}
	    		\caption{Laplace}
	    	\end{subfigure}
	\vspace{-4pt}
	\caption{EIG surfaces estimated by four methods for the two-dimensional design $(t_1, t_2)$ for the continuous time model described in Sec.~\ref{sec:deathprocess}. 
		The optimal design $(t_1^*, t_2^*)$ determined by each method is indicated with a cross. The posterior method with a LogNormal variational distribution yields nearly exact results. The posterior method with a Truncated Normal distribution and the Laplace method are not as accurate but still result in designs with large EIG.
		Note that the EIG has been scaled for interpretability and that all four figures use a common scale. The errors of these estimators are examined more closely in Figure~\ref{fig:diseasediff}.
		\label{fig:disease}}
\end{figure*}

\begin{figure*}[t]
  \centering
  	\begin{subfigure}[b]{0.31\textwidth}
  		\centering
  		\includegraphics[width=\textwidth]{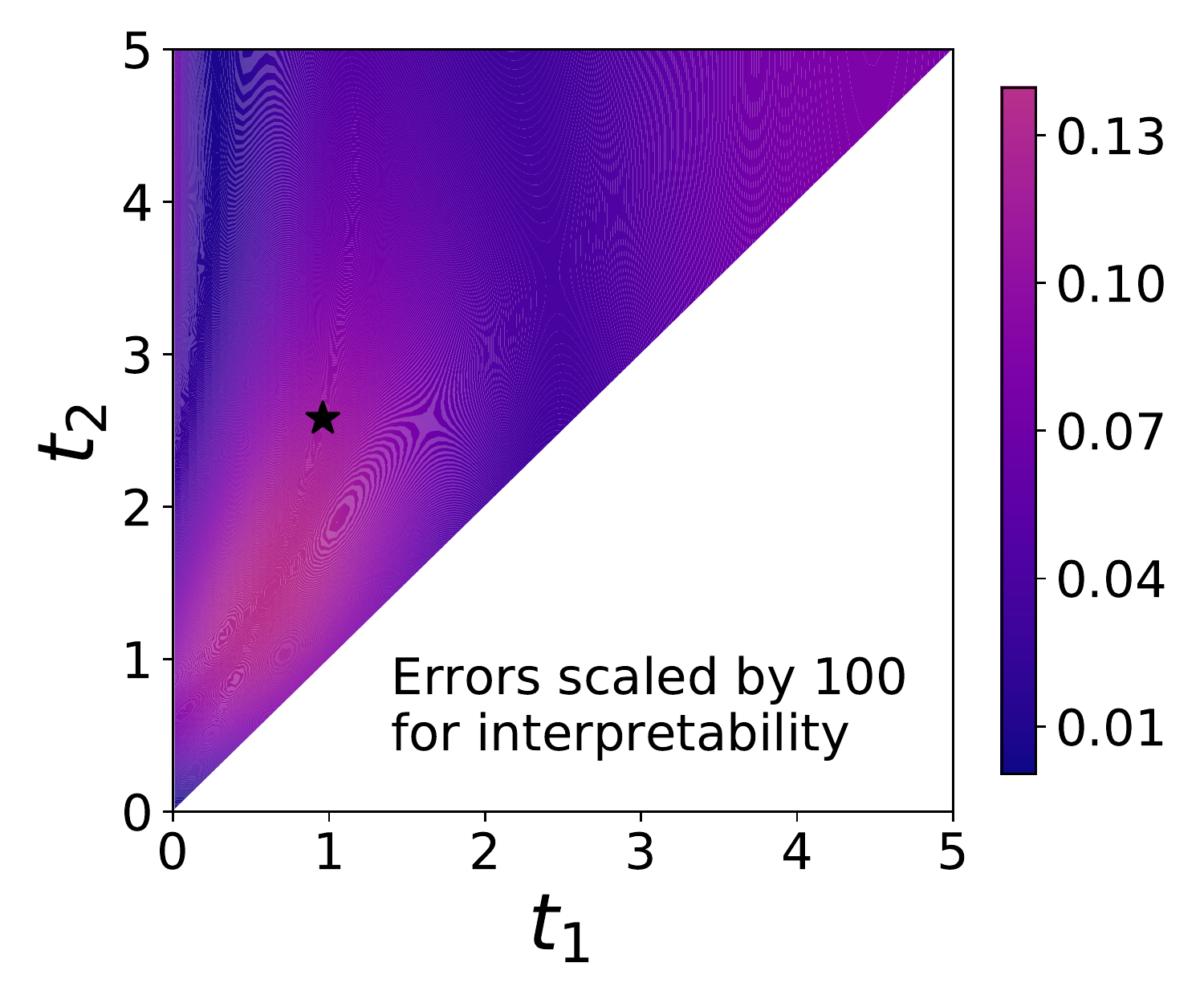}
  		\caption{Posterior LogNormal}
  	\end{subfigure}
  	~~
    	\begin{subfigure}[b]{0.31\textwidth}
    		\centering
    		\includegraphics[width=\textwidth]{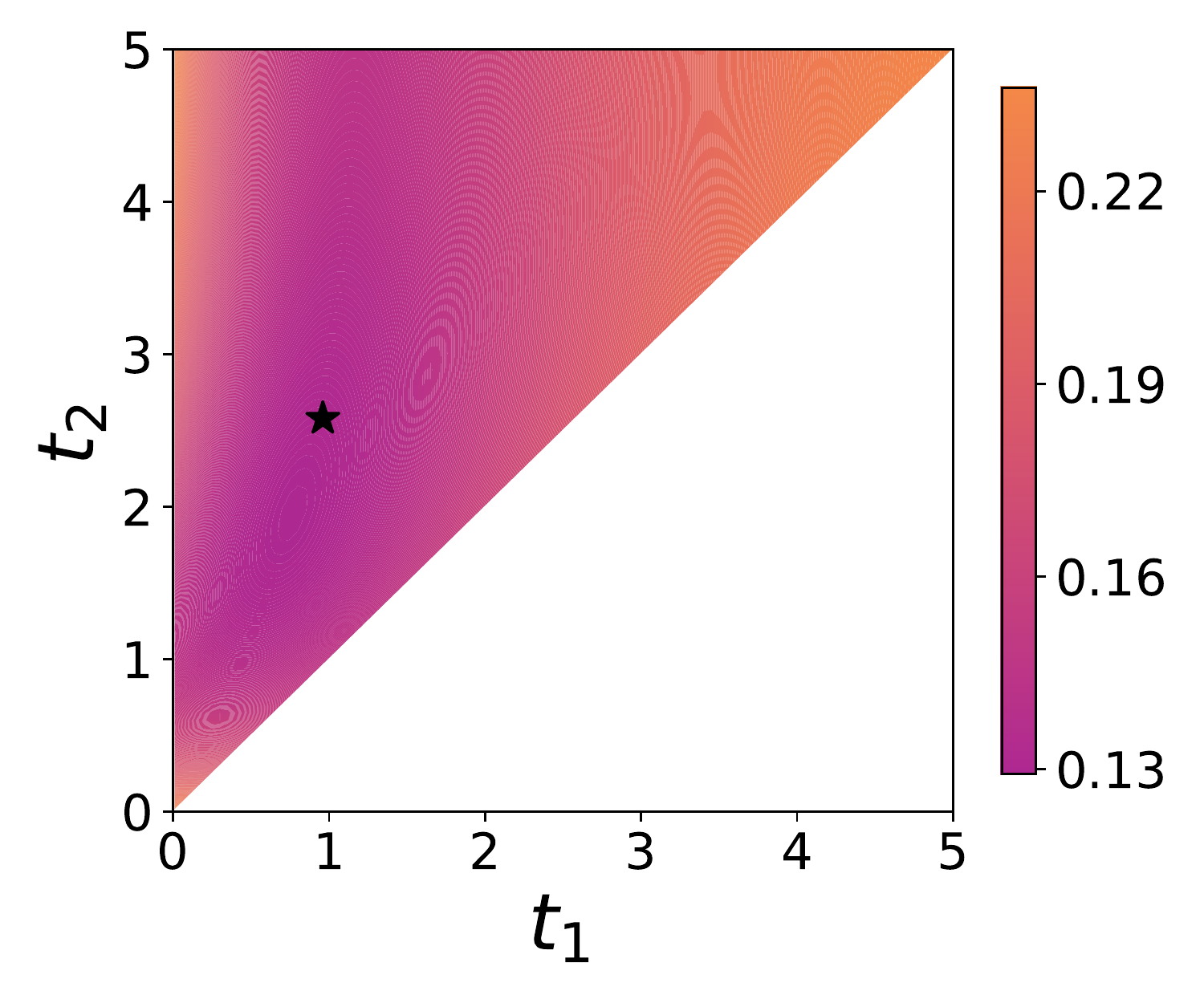}
    		\caption{Posterior Truncated Normal}
    	\end{subfigure}
    	~~
    	\begin{subfigure}[b]{0.31\textwidth}
    		\centering
    		\includegraphics[width=\textwidth]{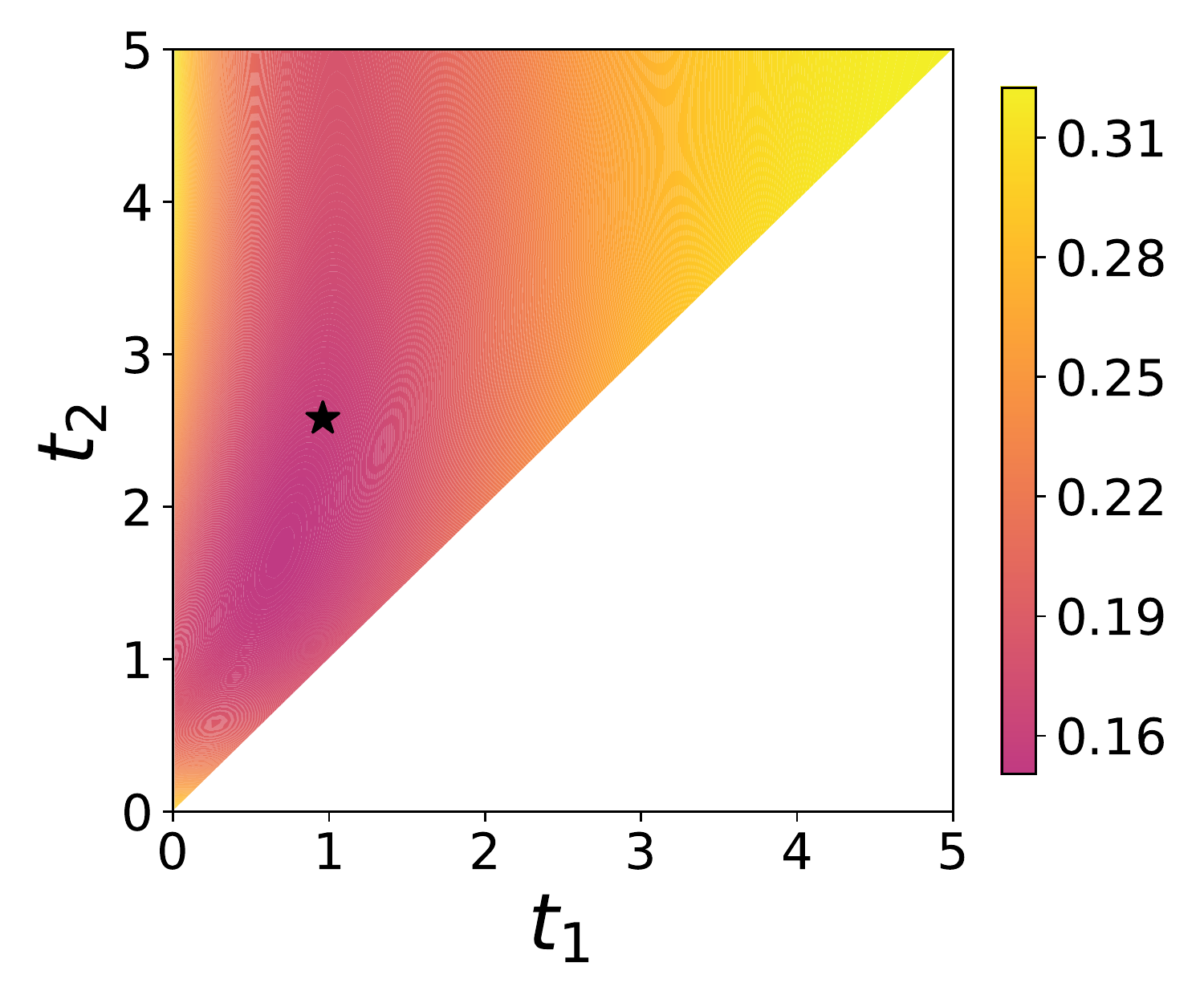}
    		\caption{Laplace}
    	\end{subfigure}
  \caption{Absolute EIG errors corresponding to the estimates depicted in Fig.~\ref{fig:disease}. 
  The optimal design $(t_1^*, t_2^*)$ determined by an exact method is indicated with a star.
  The absolute error of the LogNormal Posterior estimate is $\sim10^{-3}$ across the design space.
  The mean absolute error of the Laplace EIG estimates across the design space is about 30$\%$ higher than for the Posterior method with a Truncated Normal variational distribution.
  In this case the Laplace method results in an upper bound, while (as always) both Posterior methods yield a lower bound.
  All three figures have the same scale as Fig.~\ref{fig:disease}, except for the LogNormal errors, which have been scaled by an additional factor of 100.}
    \label{fig:diseasediff}
\end{figure*}

We examine experimental design for the simple continuous time process considered in ~\citep{cook2008optimal} and~\citep{kleinegesse2018efficient}, arising in epidemiology.
Consider a population with fixed size $N$ that is initially healthy at time $t=0$, with individuals becoming infected at a constant rate $b$ as time evolves.
We consider a design space $d = (t_1, t_2)$, where $0 \le t_1 \le t_2$, corresponding to the times at which we measure the number of infected individuals. 
We place a log-normal prior on the infection rate $b$.

For this example, we investigate how the choice of variational family affects the asymptotic bias. In Fig.~\ref{fig:disease} we compare the EIG surfaces obtained using four estimators: i) an exact method that uses brute force quadrature; ii) $\ipost$ with a log-normal variational distribution; iii) $\ipost$ with a truncated normal variational distribution; and iv) the Laplace approximation $\ilap$. The log-normal family matches the true posterior best, giving mean absolute errors $\sim 10^{-3}$. The second posterior method and the Laplace approximation both make the same distributional assumption, but Laplace results in absolute errors that are about 30\% higher
than for the posterior method.
See Fig.~\ref{fig:diseasediff} for a closer analysis of the errors of the approximate methods.

\paragraph{Experimental details}

The likelihood for observing $(I_1, I_2)$ infected individuals from a population of size $N$ at times $(t_1, t_2)$ is given by \citep{ehrenfeld1962some}:
\begin{align}
\nonumber
    p(I_1, I_2 | b, t_1, t_2) = &\frac{N!}{I_1! (I_2-I_1)!(N-I_2)!} \left[ 1 - e^{-b t_1} \right]^{I_1} 
\times \\ & \left[ 1 - e^{-b (t_2-t_1)} \right]^{I_2-I_1}
\left[e^{-b t_1} \right]^{I_2-I_1} \left[e^{-b t_2} \right]^{N-I_2}
\end{align}
The prior over the infection rate $b > 0$ is taken to be
\begin{align}
\log b \sim N(\mu_b, \sigma_b) 
\end{align}
so that the joint density is given by
\begin{align}
p(I_1, I_2, b | t_1, t_2) = p(I_1, I_2 | b, t_1, t_2)  p(b) 
\end{align}
In our experiment we choose $N=10$, $\mu_b=0$, and $\sigma_b=0.25$. The figures are scaled such that the maximum EIG over the design space (as computed with the exact method) is 1.0. For all four EIG estimation methods we use quadrature and exact summation over the outcomes $(I_1, I_2)$ where appropriate to obtain maximally accurate results. That is, the obtained results are only constrained by the methods themselves and not the computational budget used. Note that we do not make use of any kind of amortization.

\section{Consistent \acrshort{EIG} estimation with control variates}
\label{sec:cv}

In this section, we show that an approximation to the marginal density $\qm(y|d)$ can be used a control variate. Control variates are a means to reduce the variance of Monte Carlo estimators by using expectations which can be computed analytically. Here, we assume that, for every $\theta$, the KL divergence $\kl{p(y|\theta,d)}{\qm(y|d)}$ can be computed analytically. For example, this would be the case if both $p(y|\theta,d)$ and $\qm(y|d)$ were Gaussian. 

We begin by writing the EIG as
\begin{align}
	\eig(d) &= \E_{p(y, \theta|d)}\left[\log \frac{p(y|\theta,d)}{p(y|d)} \right] \\
	&= \E_{p(y, \theta|d)}\left[\log \frac{p(y|\theta,d)}{\qm(y|d)} \right] + \E_{p(y, \theta|d)}\left[\log \frac{\qm(y|d)}{p(y|d)} \right] \\
	&= \E_{p(\theta)}\left[\kl{p(y|\theta,d)}{\qm(y|d)} \right] - \kl{p(y|d)}{\qm(y|d)}.
\end{align}
We can now use our assumption on the first term, 
\begin{align}
\mathbb{E}_{p(\theta)} \left[ \kl{p(y|\theta,d)}{\qm(y|d)} \right ] \rightarrow 
\mathbb{E}_{p(\theta)} \left[ \rm{analytic\;function\;of}\;\theta \right ]
\end{align}
and this expectation can be computed efficiently with conventional Monte Carlo. For the second term, we use Nested Monte Carlo
\begin{equation}
	\label{eq:control_variate}
	\kl{p(y|d)}{\qm(y|d)} \approx \frac{1}{N}\sum_{n=1}^N \log \frac{\frac{1}{M}\sum_{m=1}^M p(y_n|\theta_m,d)}{\qm(y_n|d)}
\end{equation}
where $y_n\iid p(y|d)$ and $\theta_m \iid p(\theta)$. The key benefit of this approach is that this estimator may have lower variance than a direct NMC estimator of $\eig(d)$. Indeed, if we let $A = \log \left(\frac{1}{M}\sum_{m=1}^M (y|\theta_m,d)\right)$ and $B = \log \qm(y|d)$ then the variance of the estimator in \eqref{eq:control_variate} is
\begin{equation}
	\text{Var}(A - B) = \text{Var}(A) + \text{Var}(B) - 2\,\text{Cov}(A, B)
\end{equation}
so the variance will be low when $\text{Cov}(A, B)$ is large. We can expect this to happen when $\qm(y|d)$ is a good approximation to the true marginal density $p(y|d)$.

Finally, note that just like $\ivnmc$, this estimator
is consistent, i.e.~it will converge to the EIG as $N, M \rightarrow \infty$.

\section{$\kl{q}{p}$ versus $\kl{p}{q}$}
\label{sec:reverseforward}

\begin{figure}[t]
  	    		
  	    		\begin{subfigure}[b]{0.54\textwidth}
  	    		\centering{\includegraphics[width=\textwidth]{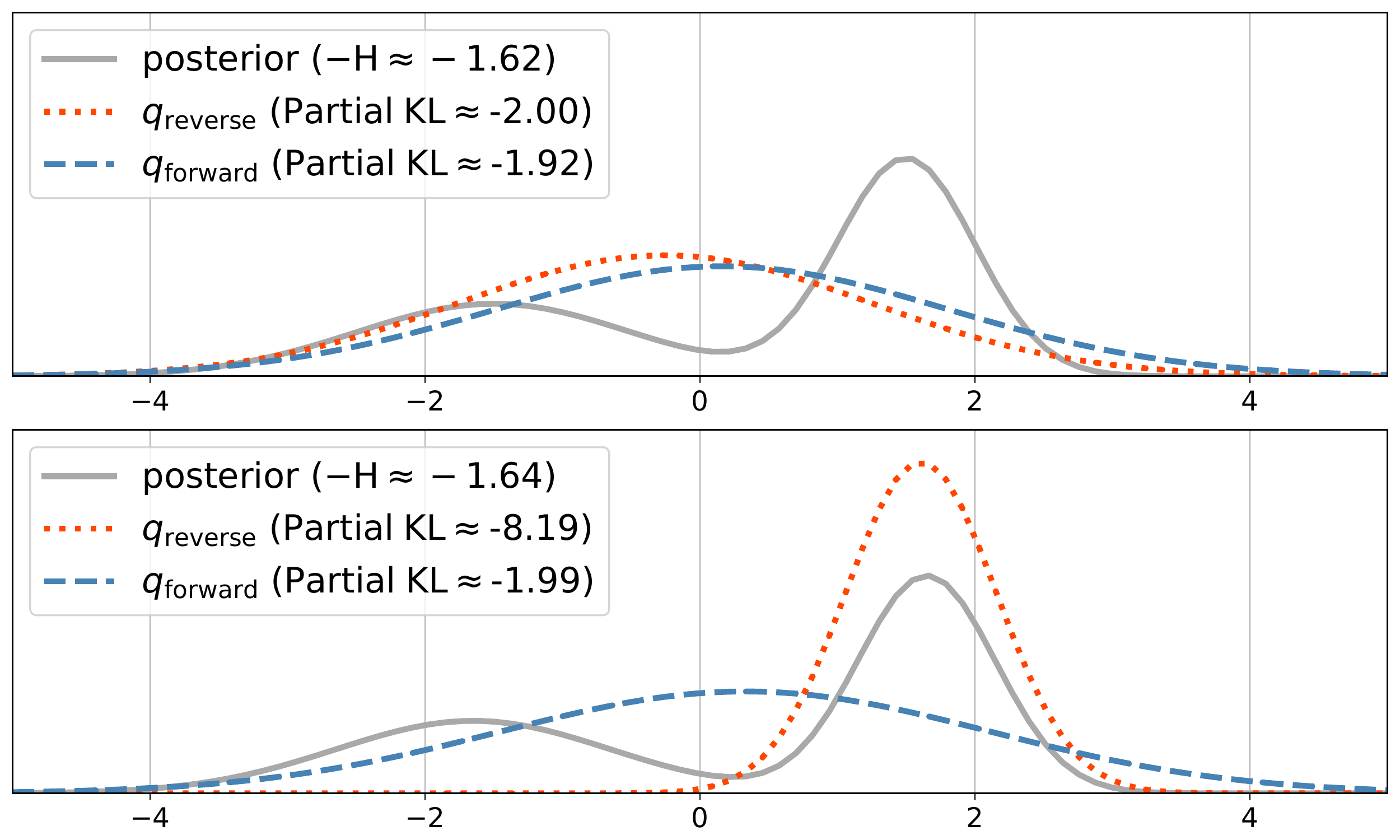}}
  	    		\caption{Forward and reverse KL minimization}
  	    		\end{subfigure}
  	    		\begin{subfigure}[b]{0.45\textwidth}
  	    		\centering{\includegraphics[width=\textwidth]{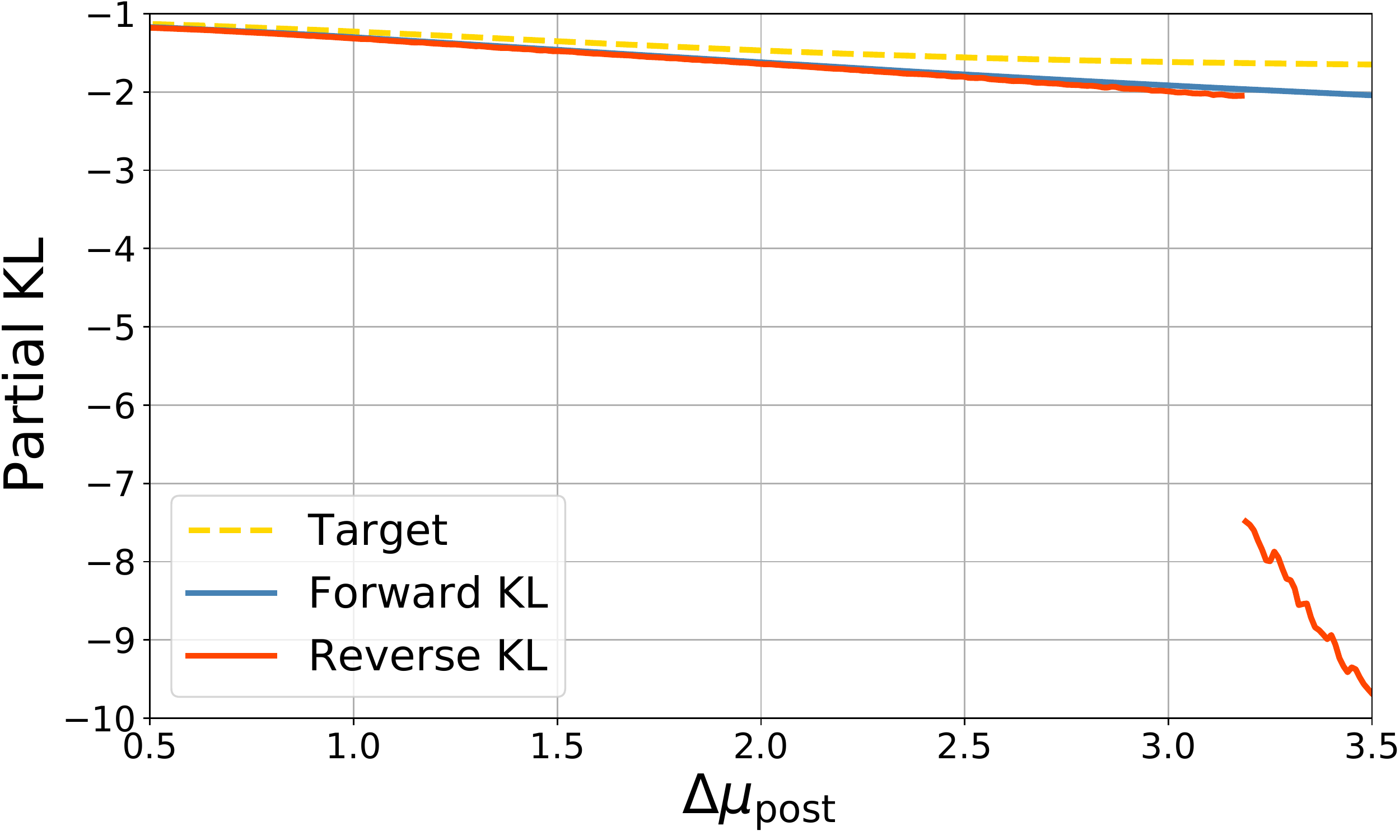}}
  	    		\caption{ Discontinuity in the partial KL }
  	    		\end{subfigure}

  	    		\caption{(a) Normal variational distributions found by fitting to a target posterior that is a mixture with two distinct Normal components. In both plots, the target posterior is a mixture of $N(\mu_1, 0.5^2)$ and $N(\mu_2, 1.0^2)$ and we vary $\Delta \mu_{\rm post} = \mu_1 - \mu_2$. In the top plot, the gap between the two components is $\Delta \mu_{\rm post} = 3.0$, while in the bottom plot $\Delta \mu_{\rm post} = 3.3$. In contrast to the behaviour resulting from forward KL minimization, the mode-seeking behaviour of reverse KL minimization leads to a large change in the corresponding optimal variational distribution from top to bottom.
  	    		(b) We plot the partial KL as we vary $\Delta \mu_{\rm post}$ for the target posterior described in (a). The partial KL as estimated by reverse KL minimization exhibits a sharp discontinuity as the gap between the two components crosses $\Delta \mu_{\rm post} \approx 3.18$. \label{fig:klexploration}}

%
  	    		
\end{figure}

In Appendix~\ref{sec:app:ipost}, we showed that our posterior estimator is implicitly minimizing the following expected KL divergence
\begin{equation}
\label{eq:posterior_kl_error}
	\eig(d) - \Ipost(d) = \E_{p(y|d)}\left[ \kl{p(\theta|y,d)}{\qp(\theta|y,d)} \right].
\end{equation}
In variational inference, the inner KL divergence is referred to as the \textit{forward} KL.
In this section, we compare our approach with a similar approach which also uses a posterior approximation, but instead minimize the \textit{reverse} KL divergence, $\kl{\qp(\theta|y,d)}{p(\theta|y,d)}$.


Specifically, we explore how the reverse KL divergence exhibits discontinuous behaviour that could be problematic in the context of EIG estimation. We begin by writing the posterior estimator as
\begin{equation}
	\Ipost(d) = \E_{p(y|d)}\left[\E_{p(\theta|y)}[\log \qp(\theta|y,d)] \right] + H[p(\theta)].
\end{equation}
The term involving $\qp$ is the expectation of the partial KL, $\mathbb{E}_{p(\theta|y)} \left[ \log \qp(\theta|y,d) \right]$. We will show that reverse KL minimization can lead to a discontinuity in the partial KL.

We consider two possible methods for choosing $\qp$. We know from \eqref{eq:posterior_kl_error} that the optimal choice of $\qp$ within a variational family $\mathcal{Q}$ is
\begin{equation}
\label{eq:qforward}
q_{\rm forward}(\theta|y,d) \triangleq \argmin_{q \in \mathcal{Q}}  \kl{p(\theta | y, d)}{q(\theta)}.
\end{equation}
An alternative choice is
\begin{equation}
\label{eq:qreverse}
q_{\rm reverse}(\theta|y,d) \triangleq \argmin_{q \in \mathcal{Q}}  \kl{q(\theta)}{ p(\theta | y, d) }
\end{equation}
which is the form usually seen in variational inference. The posterior method outlined in Section~\ref{sec:method} attempts to learn $q_{\rm forward}$ for each $y$ by maximizing the bound $\Ipost$. In this appendix, we show that the alternative $q_{\rm reverse}$, as well as resulting in less accurate EIG estimates in light of \eqref{eq:posterior_kl_error}, can lead to discontinuities in the partial KL.

Minimizing the reverse KL can result in the well-known behaviour of mode-locking---and thus mode-dropping---which in our context can result in significant misestimates of the EIG. Furthermore, since this mode-locking behaviour is discontinuous (so that it can occur for a particular design $d$ but not for a neighbouring design $d^\prime$) it can potentially result in large design-dependent bias in EIG estimation. For a quantitative exploration of this phenomenon for two bimodal posteriors and a Normal family of variational distributions $\mathcal{Q}$ see Figure~\ref{fig:klexploration}.



\end{document}